\documentclass{article}

\usepackage{microtype}
\usepackage{graphicx}
\usepackage{subfigure}
\usepackage{booktabs} 
\usepackage[graphicx]{realboxes}
\usepackage{multirow}
\usepackage{color}
\usepackage{longtable}
\usepackage{makecell}
\usepackage{newfloat}
\usepackage{enumitem}
\usepackage{hyperref}
\usepackage{amsfonts}

\usepackage[accepted]{icml2023}

\usepackage{amsmath}
\usepackage{amssymb}
\usepackage{mathtools}
\usepackage{amsthm}
\usepackage{bm}

\usepackage[capitalize,noabbrev]{cleveref}

\theoremstyle{plain}
\newtheorem{theorem}{Theorem}[section]

\newtheorem{lemma}[theorem]{Lemma}
\newtheorem{corollary}[theorem]{Corollary}
\theoremstyle{definition}

\theoremstyle{remark}

\newcommand{\argmin}{\mathop{\arg\min}}

\usepackage{subfigure}
\usepackage{graphicx}

\usepackage[textsize=tiny]{todonotes}

\begin{document}

\twocolumn[
\icmltitle{Weakly Supervised Regression with Interval Targets}



\icmlsetsymbol{equal}{*}

\begin{icmlauthorlist}
\icmlauthor{Xin Cheng}{cqu}
\icmlauthor{Yuzhou Cao}{ntu}
\icmlauthor{Ximing Li}{jlu}
\icmlauthor{Bo An}{ntu}
\icmlauthor{Lei Feng}{ntu}
\end{icmlauthorlist}

\icmlaffiliation{cqu}{College of Computer Science, Chongqing University, China}
\icmlaffiliation{ntu}{School of Computer Science and Engineering, Nanyang Technological University, Singapore}
\icmlaffiliation{jlu}{College of Computer Science and Technology, Jilin University, China}

\icmlcorrespondingauthor{Lei Feng}{lfengqaq@gmail.com}

\icmlkeywords{Machine Learning, ICML}

\vskip 0.3in
]



\printAffiliationsAndNotice{}  

\begin{abstract}
This paper investigates an interesting weakly supervised regression setting called \emph{regression with interval targets} (RIT). Although some of the previous methods on relevant regression settings can be adapted to RIT, they are not \emph{statistically consistent}, and thus their empirical performance is not guaranteed. In this paper, we provide a thorough study on RIT. First, we proposed a novel statistical model to describe the data generation process for RIT and demonstrate its validity. Second, we analyze a simple selection method for RIT, which selects a particular value in the interval as the target value to train the model. Third, we propose a statistically consistent limiting method for RIT to train the model by limiting the predictions to the interval. We further derive an estimation error bound for our limiting method. Finally, extensive experiments on various datasets demonstrate the effectiveness of our proposed method.
\end{abstract}

\section{Introduction}
\emph{Regression} is a significantly important task in machine learning and statistics \cite{stulp2015many,uysal1999overview}. The goal of the regression task is to learn a predictive model from a given set of training examples, where each training example consists of an instance (or feature vector) and a \emph{real-valued target}. Conventional supervised regression normally requires a vast amount of labeled data to learn an effective regression model with excellent performance. However, it could be difficult to obtain \emph{fully supervised} training examples due to the high cost of data labeling in real-world applications. To alleviate this problem, many weakly supervised regression settings have been investigated, such as semi-supervised regression \cite{li2017learning,wasserman2007statistical,kostopoulos2018semi}, multiple-instance regression \cite{amar2001multiple,wang2011mixture,park2020bayesian}, uncoupled regression \cite{carpentier2016learning,xu2019uncoupled}, and regression with noisy targets \cite{ristovski2010regression,hu2020simple}.

This paper investigates another interesting weakly supervised regression setting called \emph{regression with interval targets} (RIT). For RIT, we aim to learn a regression model from weakly supervised training examples, each annotated with \emph{only an interval that contains the true target value}. The learned regression model in this setting is expected to predict the target value of any test instance as accurately as possible. In many real-world scenarios, it is difficult to collect the exact true target value, while it could be easy to provide an interval in which the true target value is contained. A typical example is facial age estimation \cite{geng2013facial}. In Figure \ref{fig1}, there are two photos of Ballon d'Or King Pele at the age of 58. The fully supervised regression task requires the exact age of Pele, which is quite difficult to provide because it is common for a person to look the same over a long period of time. However, we can easily get an age interval that contains the true age of Pele. Based on the facial wrinkles, we can determine that the true age is at least 40 but not more than 70. In reality, many regression tasks face this challenge (especially in size/length/age estimation), where it is costly or impossible to obtain a true target value.
\begin{figure}[!t]
\centering
\includegraphics[width=0.47\textwidth]{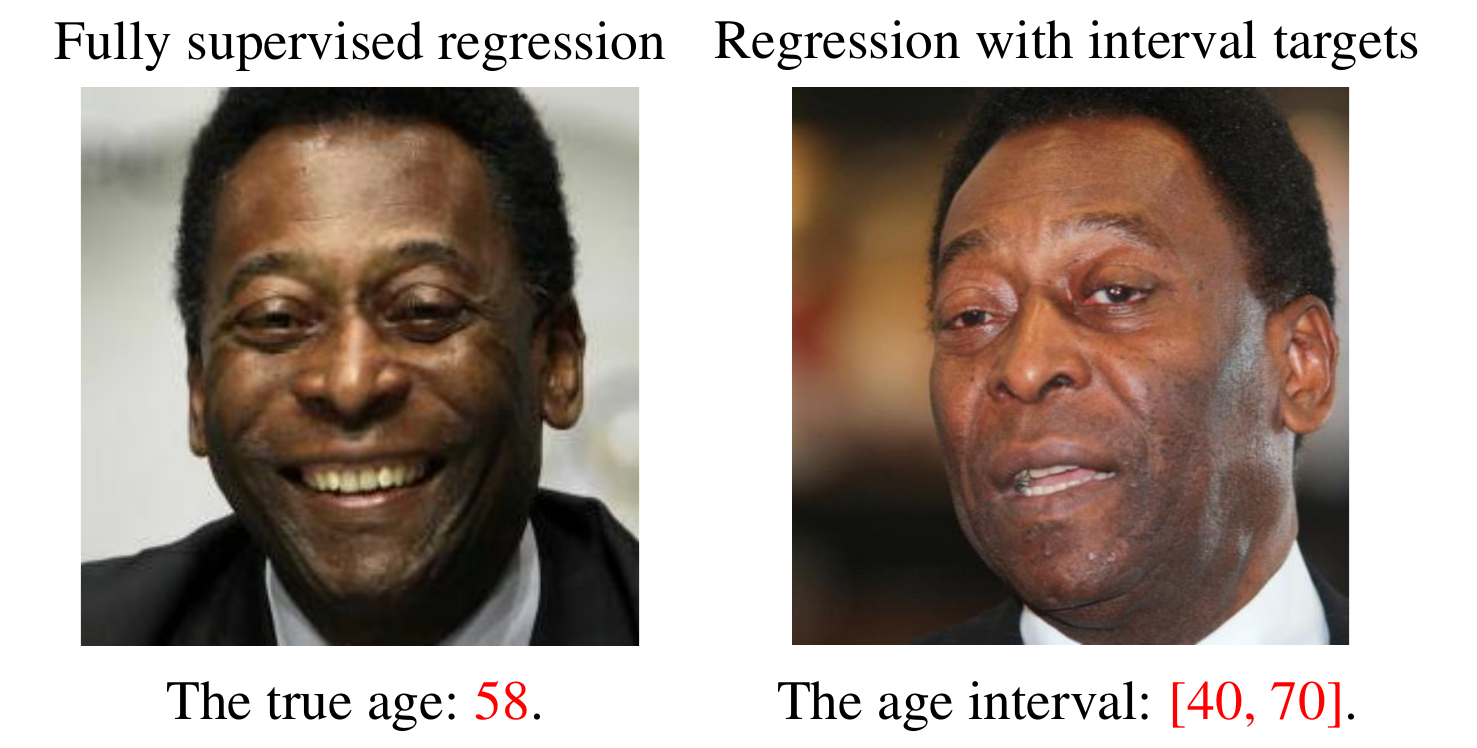}
\caption{An example of facial age estimation. Both photos were taken when Pele was 58 years old.}
\label{fig1} 
\end{figure}

Our studied RIT is highly related to \emph{interval-valued data prediction} (IVDP) \cite{ishibuchi1991extension,neto2008centre,neto2010constrained,fagundes2014interval}. IVDP allows each training example to be annotated with an interval and learns a regression model to \emph{predict the interval} containing the true target value of a test instance. There are even some studies that allow features to be intervals as well \cite{manski2002inference,yang2018l1,sadeghi2019efficient}. It is worth noting that IVDP aims to predict the interval that contains the target value, while our studied RIT aims to predict the true target value. 
Although some of the above methods could be adapted to our studied RIT settings, they are not \emph{statistically consistent} (i.e., the learned model is infinite-sample consistent to the optimal model), and thus the empirical performance is not guaranteed.

In this paper, we provide a thorough study on RIT, and the main contributions can be summarized as follows:
\begin{itemize}[leftmargin=0.4cm,topsep=-1pt]
\item We propose a novel statistical model to describe the data generation process for RIT and demonstrate its validity. Having an explicit data distribution helps us understand how data with interval targets are generated.
\item We analyze a simple selection method for RIT, which selects a particular value in the interval as the target value to train the model. We show that this intuitive method could work well if the middlemost value in the interval is taken as the target value.
\item We propose a \emph{statistically consistent} limiting method for RIT to train the model by limiting the predictions of the model to the interval. We further derive an \emph{estimation error bound} for our limiting method.
\end{itemize}
Extensive experiments on various datasets demonstrate the effectiveness of our proposed method.
\section{Related Work}
\noindent\textbf{Regression.} For the ordinary regression problem, let the feature space be $\mathcal{X}\in\mathbb{R}^d$ and the label space be $\mathcal{Y}\in\mathbb{R}$. Let us denote by $(\bm{x},y)$ an example including an instance $x$ and a real-valued true label $y$. Each example $(\bm{x},y)\in\mathcal{X}\times\mathcal{Y}$ is assumed to be independently sampled from an unknown data distribution with probability density $p(\bm{x},y)$. For the regression task, we aim to learn a model $f:\mathcal{X}\mapsto\mathbb{R}$ that tries to minimize the following expected risk:
\begin{gather}
\label{expected_regression}
R(f)=\mathbb{E}_{p(\bm{x},y)}[\ell(f(\bm{x}),y)],
\end{gather}
where $\mathbb{E}_{p(\bm{x},y)}$ denotes the expectation over the distribution $p(\bm{x},y)$ and $\ell:\mathbb{R}\times\mathbb{R}\mapsto\mathbb{R}_+$ is a conventional loss function (such as mean squared error and mean absolute error) for regression, which measures how well a model estimates a given real-valued label.

\noindent\textbf{Interval-valued data prediction.}
In order to consider multiple types of data, such as intervals, weights, and characters, symbolic data analysis \cite{bock1999analysis,billard2006symbolic} has been extensively investigated. 
As a specific task of symbolic data analysis, the purpose of \emph{interval-valued data prediction} (IVDP) is to learn an interval predictor from training data annotated with intervals. The challenge of IVDP is mainly to construct a model that outputs intervals (to ensure that the interval holds, e.g., $[0,10]$ instead of $[10,0]$). In statistics, \citet{billard2000regression} introduced a central tendency for interval data. \cite{lauro2000principal} introduced principal component analysis methods for interval data. In addition, by designing proper loss functions, network structures, or output constraints, neural networks can also be trained to output intervals for IVDP \cite{neto2008centre,neto2010constrained,giordani2015lasso,yang2019interval,sadeghi2019efficient}.

\noindent\textbf{Regression with interval-censored data.}
Another related setting is \emph{regression with interval-censored data} (RICD) \cite{rabinowitz1995regression,lindsey1998methods,lesaffre2005overview,sun2006statistical}, which aims to learn a survival function from interval-censored data. In a sequence of time points, a specific event (e.g., machine breakdown, disease attack, death) occurs between two time points, and the interval formed by these two time points is called interval censoring. RICD was widely used in survival analysis \cite{machin2006survival,kleinbaum2012survival,wang2019machine}. In contrast to our studied RIT setting, RICD aims to obtain a survival function to estimate the occurring probability of an event, instead of learning a predictive model. 
\section{Regression with Interval Targets}
\noindent\textbf{Notations.} Suppose the given training set is denoted by $\{(\bm{x}_i,S_i)\}_{i=1}^n$ where $S_i$ represents the interval $[\underline{y}_i,\overline{y}_i] \in \mathcal{Y}\times \mathcal{Y}$ assigned to the instance $\bm{x} \in \mathcal{X}$, and $|S_i|=\overline{y}_i-\underline{y}_i$ represents the size of the interval $S_i$. Each training example $(\bm{x}_i, S_i)$ is assumed to be sampled from an unknown joint distribution with probability density $\tilde{p}(\bm{x}, S)$. In this setting, the true label $y_i\in\mathcal{Y}$ of the instance $\bm{x}_i$ is guaranteed to be contained in the interval $S_i$. The goal of interval regression is to induce a regression model $f:\mathcal{X}\mapsto\mathbb{R}$ that can accurately predict the target value of a test instance. Interestingly, this setting can be considered as a generalized setting of ordinary regression, because we can easily convert the ordinary regression example $(\bm{x}_i,y_i)$ to an interval regression example by rewriting $y_i$ as the interval $[\underline{y}_i,\overline{y}_i]$ where $\underline{y}_i=\overline{y}_i=y_i$. In this paper, we denote by $p(\cdot)$ the probability density and $\mathrm{Pr}[\cdot]$ the occurring probability.

\noindent\textbf{Small ambiguity degree.} For ensuring that RIT is learnable (i.e., the true target value concealed in the interval is distinguishable), we assume that RIT should satisfy the small ambiguity degree condition \cite{cour2011learning}, where the ambiguity degree in our setting is defined as
\begin{gather}
\nonumber
\lambda = \sup_{(\bm{x},y)\sim p(\bm{x},y),(\bm{x},S)\sim \tilde{p}(\bm{x},S),y' \in \mathcal{Y},y' \neq y} \mathrm{Pr}[y' \in S].
\end{gather}
The ambiguity degree $\lambda$ is the maximum probability of a specific incorrect target $y'$ co-occurring with the true target $y$ in the same interval $S$. We can observe that when $\lambda=1$, the incorrect target $y'$ always appears with the true target $y$ together, and thus we can no longer distinguish which one is the true target. Therefore, the RIT setting requires to assume that the small ambiguity degree condition is satisfied (i.e., $\lambda < 1$), in order to ensure that this setting is learnable.

\subsection{Data Generation Process}\label{data_gen_sec}
To avoid the sampled intervals being unreasonable (i.e., the size is unexpectedly large), we use $q$ to denote the maximum allowed interval size $|S|$. Then, we assume that each example $(\bm{x},S)$ with $S=[\underline{y},\overline{y}]$ is independently sampled from a probability distribution with the following density:
\begin{gather}
\label{equ2}
\tilde{p}(x,S)=\int_{\underline{y}}^{\overline{y}}p(S | y)p(\bm{x},y) \mathrm{d}y,
\end{gather}
where
\begin{gather}
\label{equ3}
p(S|y)=\left\{
	\begin{aligned}
	\frac{2}{q^2}&, \quad y\in S\ \mathrm{and}\ |S|\leq q,\\
	0&, \quad \mathrm{otherwise}.
	\end{aligned}
	\right
	.
\end{gather}
In Eq.~(\ref{equ2}), we assume $p(S|y,\bm{x}) = p(S|y)$, which means that given the true label $y$, the interval $S$ is independent of the instance $\bm{x}$. Such a class-dependent and instance-independent assumption was widely adopted by many previous studies in the weakly supervised learning field \cite{patrini2017making,ishida2019complementary,feng2020provably}. In Eq.~(\ref{equ3}), we assume that given a specific label $y$, all possible intervals $\{S | y \in S, |S| \leq q\}$ are uniformly sampled. 

We show that our presented joint distribution $\tilde{p}(\bm{x},S)$ is a valid probability distribution by the following theorem.
\begin{theorem}
\label{thm1}
The following equality holds:
\begin{equation}
\int_{\mathcal{S}}\int_{\mathcal{X}}\tilde{p}(\bm{x},S)\mathrm{d}\bm{x}\mathrm{d}S = 1.
\end{equation}
\end{theorem}
In addition to proving that $\tilde{p}(\bm{x},S)$ is a valid probability distribution, we also need to verify that $\tilde{p}(\bm{x},S)$ meets the key requirement of RIT, i.e., the true target $y$ is guaranteed to be contained in the interval $S$ for every example $(\bm{x},S)$ sampled from $\tilde{p}(\bm{x},S)$. The following theorem provides an affirmative answer to this question.

\begin{theorem}
\label{key}
For any interval example $(\bm{x},S)$ independently sampled from the assumed data distribution $\tilde{p}(\bm{x},S)$ defined in Eq.~(\ref{equ2}), the true target $y$ is always in the interval $S$, i.e., $\mathrm{Pr}[y\in S|\bm{x},S]=1$, $\forall(\bm{x},S) \backsim\tilde{p}(\bm{x},S)$.
\end{theorem}

\subsection{Real-World Motivation}
Here, we give a real-world motivation for the assumed data distribution $\tilde{p}(\bm{x},S)$. For the data annotations in the regression task, it could be difficult to directly provide the exact true target value for each instance.
Fortunately, it would be easier if the annotation system can randomly generate an interval and ask annotators whether the true target value is contained in the generated interval or not. Given an instance $\bm{x}$, the maximum size $q$ of the interval, and the maximum and minimum values $y_{\mathrm{max}}, y_{\mathrm{min}}$ of the label space $\mathcal{Y}$, suppose the annotation system randomly and uniformly samples $\underline{y}$ and $\overline{y}$ from the interval $[y_{\mathrm{min}}-q,y_{\mathrm{max}}+q]$ ($\pm q$ is to ensure that all possible $y$ have the same possible number of intervals) to generate an interval $S=[\underline{y},\overline{y}]$. If the sampled $\underline{y}$ and $\overline{y}$ satisfy the two conditions: $\overline{y} - \underline{y} \leq q$ and the true target value $y$ of $\bm{x}$ belongs to $[\underline{y},\overline{y}]$ i.e., $y \in [\underline{y},\overline{y}]$, then we collect an interval regression example $(\bm{x},S)$ where $S=[\underline{y},\overline{y}]$, otherwise we discard the interval $S$ for the instance $\bm{x}$. In this way, each collected interval regression example $(\bm{x},S)$ exactly follows the data distribution defined in Eq.~(\ref{equ2}). We will demonstrate this argument below.

We start by considering the case where the annotation system has discarded all intervals larger than a given maximum interval value $q$. Then we have the following lemma.
\begin{lemma}
\label{lem1}
Given the maximum value $q$ allowed for the interval size, and the maximum value $y_{\mathrm{max}}$ and the minimum value $y_{\mathrm{min}}$ of the label space $\mathcal{Y}$, for any instance $\bm{x}$ with its true target $y$ and any interval $S$ with size no greater than $q$ (i.e., $|S|\leq q$), the following equality holds:
\begin{gather}
\label{le}
\mathrm{Pr}[y \in S|\bm{x}] = \frac{q}{2(y_{\mathrm{max}}-y_{\mathrm{min}})+q}.
\end{gather}
\end{lemma}
In the case of no additional information, we can only choose the interval randomly, so the above probability is uniform. When the maximum value $q$ allowed in the interval increases, the probability in Eq.~(\ref{le}) will increase, which is in line with our knowledge because a larger interval is more likely to contain the true value $y$. Similarly, the larger the space ($y_{\mathrm{max}}-y_{\mathrm{min}}$) allowed for sampling, the more difficult it is to obtain an interval containing the true label $y$. Based on lemma \ref{lem1}, we have the following theorem.
\begin{theorem}
\label{thm2}
Under the same setting of Lemma \ref{lem1}, the distribution of collected data where the true label $y \in \mathcal{Y}$ of an example $\bm{x}$ belongs to the interval $S$ is the same as Eq.~(\ref{equ2}), i.e.,
\begin{gather}
\nonumber
p(x,S | y \in S) = \tilde{p}(x,S).
\end{gather}
\end{theorem}
Theorem \ref{thm2} clearly demonstrates that our assumed data distribution $\tilde{p}(\bm{x},y)$ exactly accords with the real-world motivation introduced above.
\section{The Proposed Methods}
In this section, we introduce a simple method that selects a \emph{particular value} in the interval as the target value to train a regression model. This method is simple and intuitive, but it only considers a single value in the interval and ignores the overall interval information. To overcome this drawback, we propose a limiting method that limits the predicted value of the model to be in the interval, which is \emph{statistically consistent} under a very mild condition.

\subsection{The Simple Selection Method}\label{select_sec}
Given an interval, an intuitive solution is to select a particular value in the interval as the target value:
\begin{gather}
\label{one}
\ell_{\mathrm{sel}}(f(\bm{x}), S) =\ell(f(\bm{x}), y^\prime), \mathrm{where}\ y^\prime \in S.
\end{gather}
As shown in Eq.~(\ref{one}), this method aims to select one value in the interval as the target value and regard the loss of this value as the predictive loss for the interval regression example $(\bm{x}, S)$. This simple method has an obvious drawback, i.e., selecting only one value in the interval ignores the influence of other values in the interval. Intuitively, the selection strategy has a significant impact on the final performance of the trained model. Here, we provide three typical strategies to select a particular value in the interval:
\begin{itemize}
\item Selecting the leftmost value:
\begin{gather}
\label{left}
\ell_{\mathrm{left}}(f(\bm{x}), S) =\ell(f(\bm{x}), \underline{y}).
\end{gather}
\item Selecting the rightmost value:
\begin{gather}
\label{right}
\ell_{\mathrm{right}}(f(\bm{x}), S) =\ell(f(\bm{x}), \overline{y}).
\end{gather}
\item Selecting the middlemost value:
\begin{gather}
\label{midd}
\ell_{\mathrm{mid}}(f(\bm{x}), S) =\ell(f(\bm{x}), \frac{\underline{y}+\overline{y}}{2}).
\end{gather}
\end{itemize}

Obviously, these three strategies select the three most particular values (including the leftmost value, the rightmost value, and the middlemost value) in the interval. Different selection strategies could result in different errors of estimating the true target value.
Given any interval example $(\bm{x},S)$ with $|S| = a$, we analyze the mean absolute error of estimating the true target $y$ by the three strategies when $y$ falls at any position in the interval. We illustrate this analysis in Table \ref{error_analysis}. As shown in Table \ref{error_analysis}, since any value in the interval could be the true target, we calculate the maximum error and the expected error for each selection strategy.
\begin{table}[!t]
\caption{Error analysis for three selection strategies (with $|S|=a$).}
\label{error_analysis}
\resizebox{0.48\textwidth}{!}{
\setlength{\tabcolsep}{1.5mm}{
\begin{tabular}{c|ccc}
\toprule
Strategy   & Selected value                      & Maximum error & Expected error \\ \midrule
Leftmost   &   $\underline{y}$               & $a$          & $a/2$          \\ 
Rightmost  &   $\overline{y}$               & $a$          & $a/2$          \\  
Middlemost & ${(\underline{y}+\overline{y})}/{2}$ & $a/2$        & $a/4$         
\\
\bottomrule
\end{tabular}
}
}
\end{table}
Clearly, if the true target is the rightmost value (i.e., $y=\overline{y}$), the rightmost selection strategy is optimal and the errors ($a/2$ and $a/4$) of the leftmost and middlemost selection strategies are maximum. If the true target is the middlemost value (i.e., $y=(\underline{y}+\overline{y})/2$), the middlemost selection strategy is optimal and both the leftmost and rightmost selection strategies achieve the error of $a/2$. If the true target is the leftmost value (i.e., $y=\underline{y}$), the leftmost selection strategy is optimal and the errors ($a/2$ and $a/4$) of the rightmost and middlemost selection strategies are maximum. When all the values in the interval have the same probability of being the true target, the expected error of both the leftmost and rightmost selection strategies is $a/2$ and the expected error of the middlemost selection strategy is $a/4$. According to the above analysis, we can find that the middlemost selection strategy is relatively stable and can achieve a smaller error regardless of the true target value. Therefore, the middlemost selection strategy is expected to achieve better performance than the leftmost and the rightmost selection strategies, and our empirical results in Section \ref{experiment_sec} also support this argument.

\noindent\textbf{Further discussion.} In addition to the above middlemost selection strategy, it is natural to consider another strategy from the loss perspective, i.e., the average loss of $\ell_{\mathrm{left}}$ and $\ell_{\mathrm{right}}$. Specifically, for each interval example $(x, S)$, we can define the average loss as $\ell_{\mathrm{avgl}}(f(\bm{x}), S) = \frac{1}{2}\ell(f(\bm{x}), \underline{y}) + \frac{1}{2}\ell(f(\bm{x}), \overline{y})$. We will theoretically analyze this method and show that only with a specific choice of the regression loss $\ell$ (i.e., mean absolute error), $\ell_{\mathrm{avgl}}(f(\bm{x}), S)$ can achieve good empirical performance with theoretical guarantees.

\subsection{The Statistically Consistent Limiting Method}
We can find that the simple selection method only considers a single value in the interval and ignores the overall interval information. To overcome this drawback, we propose the following limiting method that limits the predicted value in the interval:
\begin{gather}
\label{lmfun}
\ell_{\mathrm{LM}}(f(\bm{x}), S) = \mathbb{I}\left[\underline{y}-f(\bm{x})>0\right] + \mathbb{I}\left[f(\bm{x})-\overline{y}>0\right].
\end{gather}
This loss function takes value 0 if $\underline{y}\leq f(\bm{x})\leq \overline{y}$, otherwise 1. This is in line with our intention to limit the predicted values in the interval.
Then, the expected regression risk of our proposed limiting method can be represented as follows:
\begin{align}
\label{risk_lm}
R_{\mathrm{LM}}(f)&=\mathbb{E}_{\tilde{p}(\bm{x}, S)}[\ell_{\mathrm{LM}}(f(\bm{x}),S)].
\end{align}
We demonstrate that our proposed limiting method is \emph{model-consistent}, i.e., the model learned by the limiting method from interval data converges to the optimal model learned from fully supervised data. In particular, we assume that the hypothesis space $\mathcal{F}$ is strong enough \cite{lv2020progressive} such that the optimal model (i.e., $f^\star = \argmin_{f\in\mathcal{F}}{R}(f)$) in the hypothesis space makes the optimal risk equal to 0 (i.e., $R(f^\star)=0$). Then we introduce the following theorem.

\begin{theorem}
\label{op-lm}
Suppose that the hypothesis space $\mathcal{F}$ is strong enough (i.e., $f^\star = \mathrm{argmin}_{f\in\mathcal{F}}{R}(f)$ leads to $R(f^{\star})=0$). The model $f^\star_{\mathrm{LM}}=\mathrm{argmin}_{f\in\mathcal{F}}R_{\mathrm{LM}}(f)$ learned by our limiting method is equivalent to the optimal model $f^\star = \mathrm{argmin}_{f\in\mathcal{F}}{R}(f)$.
\end{theorem}
Theorem \ref{op-lm} demonstrates that the optimal regression model learned from fully labeled data can be identified by our limiting method given only data with interval targets (i.e., our limiting method is model consistent). However, we cannot directly train a regression model by using our limiting method in Eq.~(\ref{lmfun}), since the loss function in Eq.~(\ref{lmfun}) is non-convex and discontinuous. To address this problem, we propose the following surrogate loss function of our limiting method:
\begin{align}
\label{surrogate_loss}
\psi_{\mathrm{LM}}(f(\bm{x}), S) = &\max(0,\underline{y}-f(\bm{x}))+\max(0,f(\bm{x})-\overline{y}).
\end{align}
As can be seen from Eq.~(\ref{surrogate_loss}), this surrogate loss is convex and is an upper bound of the original loss in Eq.~(\ref{lmfun}). With the surrogate loss in Eq.~(\ref{surrogate_loss}), the expected regression risk of our proposed limiting method can be represented as:
\begin{align}
\nonumber
R_{\mathrm{LM}}^{\psi}(f)&=\mathbb{E}_{\tilde{p}(\bm{x}, S)}[\psi_{\mathrm{LM}}(f(\bm{x}),S)].
\end{align}
Then, we demonstrate that our limiting method with the surrogate loss is still consistent, by the following theorem.
\begin{theorem}
\label{lm-su}
Suppose the hypothesis space $\mathcal{F}$ is strong enough (i.e., $f^\star = \mathrm{argmin}_{f\in\mathcal{F}}{R}(f)$ leads to $R(f^{\star})=0$).
The model $f^{\psi\star}_{\mathrm{LM}}=\mathrm{argmin}_{f\in\mathcal{F}}R_{\mathrm{LM}}^{\psi}(f)$ learned by the surrogate method is equivalent to the optimal model $f^\star = \mathrm{argmin}_{f\in\mathcal{F}}{R}(f)$.
\end{theorem}
Theorem \ref{lm-su} shows that the model learned by our limiting method is also equivalent to the optimal model $f^\star$ (learned from fully labeled data). This indicates that using the surrogate loss in Eq.~(\ref{surrogate_loss}), our limiting method is still consistent. Therefore, we can learn an effective regression model from the given dataset $\{\bm{x}_i,S_i\}_{i=1}^n$ by directly minimizing the following empirical risk:
\begin{gather}
\widehat{R}_{\mathrm{LM}}^{\psi}(f) = \sum\nolimits_{i=1}^n\psi_{\mathrm{LM}}(f(\bm{x}), S).
\end{gather}
Here, we further relate our limiting method $\psi_{\mathrm{LM}}$ to the average method $\ell_{\mathrm{avgl}}$ discussed in Section \ref{select_sec}, by the following corollary.
\begin{corollary}
\label{mae_surr}
The same minimizer (model) can be derived from $\psi_{\mathrm{LM}}$ and $\ell_{\mathrm{avgl}}$ if the mean absolute error is used as the regression loss $\ell$ in $\ell_{\mathrm{avgl}}$.
\end{corollary}
Corollary \ref{mae_surr} implies that with the mean absolute error, the average loss $\ell_{\mathrm{avgl}}$ is also model-consistent, because our limiting method is model-consistent. However, using other losses (e.g., the mean squared error) cannot make $\ell_{\mathrm{avgl}}$ theoretically grounded, and thus the empirical performance is guaranteed. We conduct experiments to demonstrate this argument, and experimental results (given in Appendix \ref{comavgl}) show that the mean absolute error clearly outperforms the mean squared error, when used in $\ell_{\mathrm{avgl}}$.

\noindent\textbf{Consistency analysis.}
Here, we provide a consistency analysis for the above limiting method, which shows that the model $\widehat{f}_{\mathrm{LM}}^{\psi}=\argmin_{f\in\mathcal{F}}\widehat{R}_{\mathrm{LM}}^{\psi}(f)$ (empirically learned from RIT data by using our limiting method) is infinite-sample consistent to the optimal model $f^\star$.
\begin{theorem}
\label{error_bound}
Assume that for all $(\bm{x},S)$ with $S=[\underline{y}, \overline{y}]$ drawn from $\tilde{p}(\bm{x},S)$ and all $f\in\mathcal{F}$, there exist constants $M$ and $M^\prime$ such that
$\max(\underline{y} - f(\bm{x}), 0) \leq M$ and $\max(f(\bm{x}) - \overline{y}, 0) \leq M^\prime$. Suppose that the pseudo-dimensions of $\{(\bm{x}, \underline{y}) \mapsto \max(\underline{y} - f(\bm{x}), 0)\mid f\in\mathcal{F}\}$ and $\{(\bm{x}, \overline{y}) \mapsto \max(f(\bm{x}) - \overline{y}, 0)\mid f\in\mathcal{F}\}$ are finite, which are denoted by $d$ and $d^\prime$.  Then, with probability at least $1-\delta$,
\begin{align}
\nonumber
&R^{\psi}_{\mathrm{LM}}(\widehat{f}_{\mathrm{LM}}^{\psi}) - R^{\psi}_{\mathrm{LM}}(f^\star) \leq 2M\sqrt{\frac{2d\log\frac{en}{d}}{n}}\\
\nonumber
&\qquad\qquad\quad + 2M^\prime\sqrt{\frac{2d^\prime\log\frac{en}{d^\prime}}{n}} + 2(M+M^\prime)\sqrt{\frac{\log\frac{4}{\delta}}{2n}}.
\end{align}
\end{theorem}
Theorem \ref{error_bound} shows that the risk of $\widehat{f}^{\psi}_{\mathrm{LM}}$ converges to the risk of $f^\star$, as the number of training data goes to infinity.


\begin{table*}[!t]
\centering
\caption{Test performance (mean and std) of each method on AgeDB. The used evaluation metrics include MSE and MAE. We repeat the sampling-and-training process 3 times. The best performance is highlighted in bold.}
\label{agedb}
\resizebox{0.99\textwidth}{!}{
\setlength{\tabcolsep}{3.7mm}{
\begin{tabular}{cc|ccccc|ccccc}
\toprule
\multicolumn{2}{c|}{Metric}                    & \multicolumn{5}{c|}{MSE}                                                                & \multicolumn{5}{c}{MAE}                                                                 \\ \hline
\multicolumn{2}{c|}{Approach}                   & $q=30$            & $q=40$            & $q=50$            & $q=60$            & $q=70$            & $q=30$            & $q=40$            & $q=50$            & $q=60$           & $q=70$            \\ \hline
\multirow{6}{*}{Leftmost} & \multirow{2}{*}{MAE} & 158.38 & 210.34 & 205.75 & 283.83 & 355.59 & 9.88 & 11.59 & 11.27 & 13.62 & 14.93 \\
 &  & (6.71) & (19.85) & (22.04) & (15.36) & (105.16) & (0.22) & (0.61) & (0.80) & (0.62) & (2.38) \\
 & \multirow{2}{*}{MSE} & 134.99 & 196.83 & 221.54 & 295.49 & 347.27 & 9.25 & 11.19 & 11.91 & 13.90 & 14.98 \\
 &  & (4.03) & (19.28) & (18.04) & (32.09) & (80.77) & (0.06) & (0.71) & (0.48) & (1.24) & (2.10) \\
 & \multirow{2}{*}{Huber} & 156.34 & 175.95 & 208.03 & 317.17 & 360.11 & 9.85 & 10.54 & 11.42 & 14.39 & 15.41 \\
 &  & (17.95) & (7.61) & (18.36) & (22.57) & (53.12) & (0.57) & (0.19) & (0.61) & (0.37) & (1.30) \\ \hline
\multirow{6}{*}{Rightmost} & \multirow{2}{*}{MAE} & 154.87 & 196.95 & 233.18 & 255.24 & 428.05 & 9.91 & 11.31 & 12.48 & 13.12 & 17.51 \\
 &  & (13.12) & (20.96) & (26.58) & (15.12) & (41.79) & (0.45) & (0.58) & (0.79) & (0.40) & (0.93) \\
 & \multirow{2}{*}{MSE} & 146.85 & 215.06 & 260.37 & 304.71 & 452.61 & 9.57 & 11.96 & 13.27 & 14.47 & 18.15 \\
 &  & (24.39) & (23.92) & (15.32) & (49.88) & (45.87) & (0.86) & (0.66) & (0.49) & (1.29) & (1.15) \\
 & \multirow{2}{*}{Huber} & 149.14 & 179.72 & 246.29 & 279.70 & 436.29 & 9.72 & 10.84 & 12.88 & 13.90 & 17.50 \\
 &  & (7.74) & (10.57) & (16.02) & (17.45) & (78.90) & (0.31) & (0.37) & (0.48) & (0.44) & (1.86) \\ \hline
\multirow{6}{*}{Middlemost} & \multirow{2}{*}{MAE} & 116.14 & 133.44 & 129.55 & 138.95 & 150.88 & 8.38 & 8.93 & 8.97 & 9.21 & 9.57 \\
 &  & (2.57) & (5.05) & (1.37) & (5.22) & (3.66) & (0.13) & (0.15) & (0.13) & (0.11) & (0.12) \\
 & \multirow{2}{*}{MSE} & 119.90 & 133.27 & 128.84 & 138.36 & 149.82 & 8.45 & 8.94 & 8.89 & 9.32 & 9.53 \\
 &  & (6.23) & (5.18) & (3.01) & (6.10) & (5.28) & (0.18) & (0.22) & (0.21) & (0.27) & (0.07) \\
 & \multirow{2}{*}{Huber} & 121.78 & 131.43 & 131.38 & 140.40 & 149.25 & 8.62 & 8.92 & 8.96 & 9.28 & 9.65 \\
 &  & (4.75) & (3.84) & (2.20) & (6.18) & (0.70) & (0.14) & (0.15) & (0.08) & (0.19) & (0.09) \\ \hline
\multicolumn{2}{c|}{\multirow{2}{*}{CRM}} & 221.66 & 303.52 & 398.50 & 523.53 & 653.81 & 12.18 & 14.57 & 17.17 & 20.11 & 22.89 \\
\multicolumn{2}{c|}{} & (1.45) & (11.12) & (14.80) & (4.63) & (17.30) & (0.07) & (0.30) & (0.40) & (0.08) & (0.09) \\
\multicolumn{2}{c|}{\multirow{2}{*}{RANN}} & 125.04 & 126.02 & 129.86 & 139.83 & 148.25 & 8.69 & 8.73 & 8.89 & 9.32 & 9.69 \\
\multicolumn{2}{c|}{} & (1.09) & (1.74) & (1.00) & (3.41) & (2.33) & (0.04) & (0.07) & (0.05) & (0.22) & (0.11) \\
\multicolumn{2}{c|}{\multirow{2}{*}{SINN}} & 218.16 & 302.06 & 404.80 & 524.74 & 649.27 & 12.11 & 14.54 & 17.41 & 19.97 & 22.70 \\
\multicolumn{2}{c|}{} & (1.53) & (6.53) & (2.70) & (7.17) & (2.94) & (0.08) & (0.24) & (0.04) & (0.13) & (0.14) \\
\multicolumn{2}{c|}{\multirow{2}{*}{IN}} & 118.85 & 133.86 & 138.58 & 147.95 & 152.34 & 8.46 & 9.00 & 9.15 & 9.53 & 9.65 \\
\multicolumn{2}{c|}{} & (5.20) & (6.46) & (4.06) & (2.82) & (9.30) & (0.17) & (0.13) & (0.09) & (0.07) & (0.21) \\ \hline
\multicolumn{2}{c|}{\multirow{2}{*}{LM}} & \textbf{115.76} & \textbf{121.24} & \textbf{123.51} & \textbf{128.04} & \textbf{129.47} & \textbf{8.36} & \textbf{8.67} & \textbf{8.75} & \textbf{8.82} & \textbf{8.92} \\
\multicolumn{2}{c|}{} & \textbf{(1.64)} & \textbf{(4.73)} & \textbf{(2.78)} & \textbf{(3.52)} & \textbf{(7.15)} & \textbf{(0.07)} & \textbf{(0.03)} & \textbf{(0.14)} & \textbf{(0.16)} & \textbf{(0.14)}
\\
\bottomrule
\end{tabular}
}
}
\end{table*}

\section{Experiments}\label{experiment_sec}
In this section, we conduct extensive experiments to validate the effectiveness of our proposed limiting method.
\subsection{Experimental Setup}
\noindent\textbf{Datasets.}
We conduct experiments on nine datasets, including two computer vision datasets (AgeDB \cite{8014984} and IMDB-WIKI \cite{rothe2018deep}), one natural language processing dataset (STS-B \cite{cer2017semeval}), and six datasets from the UCI Machine Learning Repository \cite{Dua2019UCI} (Abalone, Airfoil, Auto-mpg, Housing, Concrete, and Power-plant). Following the data distribution proposed in Section \ref{data_gen_sec}, We generated the following RIT datasets, including AgeDB-Interval at $q$ = 10, 20, 30, 40, and 50, IMDB-WIKI-Interval at $q$ = 20, 30, and 40, and STS-B-Interval at $q$ = 3.0, 3.5, 4.0, 4.5 and 5.0. For each UCI dataset, we selected two large values of $q$ to generate RIT data based on the span of the label space. The specific descriptions of used datasets with the corresponding base models and the specific hyperparameter settings are reported in Appendix~\ref{data}.

\noindent\textbf{Base models.}
For the UCI dataset, we used two models, a linear model and a multilayer perceptron (MLP), where the MLP model is a five-layer ($d$-20-30-10-1) neural network with a ReLU activation function. For the linear model and the MLP model, we use the Adam optimization method \cite{kingma2014adam} with the batch size set to 512 and the number of training epochs set to 1,000, and the learning rate for all methods is selected from $\{10^{-2},10^{-3}\}$. For both the IMDB-WIKI and AgeDB datasets, we use ResNet-50 \cite{he2016deep} as our backbone network. We use the Adam optimizer to train all methods for 100 epochs with an initial learning rate of $10^{-3}$ and fix the batch size to 256. For the STS-B dataset, we follow \citet{wang2019glue} to use the same 300D GloVe word embeddings and a two-layer 1500D (per direction) BiLSTM with max pooling to encode the paired sentences into independent vectors $u$ and $v$, and then pass $[u;v;|u-v|;uv]$ to a regressor. We also use the Adam optimizer to train all methods for 100 epochs with an initial learning rate of $10^{-4}$ and fix the batch size to 256.

\noindent\textbf{Compared methods.}
We use the leftmost, rightmost, and middlemost selection strategies analyzed in Section \ref{select_sec} as our baseline methods. Since the three methods do not rely on any loss function, we use the mean absolute error (MAE), the mean squared error (MSE), and the Huber loss (commonly used in regression tasks) as loss functions to form our baseline methods. For the Huber loss, the threshold value is selected from $\{1, 5\}$. In particular, we compare with multiple methods for interval-valued data prediction, including CRM \cite{neto2008centre}, SINN \cite{yang2012smoothing}, RANN \cite{yang2019interval}, IN \cite{sadeghi2019efficient}. Since the outputs of these methods are intervals, we use the midpoint of the interval as the predicted value.
\begin{table*}[!t]
\centering
\caption{Test performance (mean and std) of each method on IMDB-WIKI. The used evaluation metrics include MSE and MAE. We repeat the sampling-and-training process 3 times. The best performance is highlighted in bold.}
\label{imdb}
\resizebox{0.99\textwidth}{!}{
\setlength{\tabcolsep}{3.5mm}{
\begin{tabular}{cc|ccccc|ccccc}
\toprule
\multicolumn{2}{c|}{Metric}                    & \multicolumn{5}{c|}{MSE}                                                                & \multicolumn{5}{c}{MAE}                                                                 \\ \hline
\multicolumn{2}{c|}{Approach}                   & $q=40$            & $q=50$            & $q=60$           & $q=70$            & $q=80$            &  $q=40$            & $q=50$            & $q=60$            & $q=70$            & $q=80$            \\ \hline
\multirow{6}{*}{Leftmost}  & \multirow{2}{*}{MAE}   & 270.15          & 336.73          & 402.05          & 494.84          & 581.54          & 13.06           & 14.76           & 16.28           & 18.30           & 20.39           \\
                       &                        & (15.14)         & (30.92)         & (39.13)         & (28.64)         & (19.81)         & (0.53)          & (0.77)          & (1.10)          & (0.82)          & (0.46)          \\
                       & \multirow{2}{*}{MSE}   & 255.40          & 305.92          & 358.15          & 383.77          & 421.16          & 12.67           & 13.95           & 15.21           & 15.86           & 16.66           \\
                       &                        & (13.45)         & (7.33)          & (25.61)         & (16.82)         & (94.72)         & (0.40)          & (0.60)          & (0.57)          & (0.36)          & (2.36)          \\
                       & \multirow{2}{*}{Huber} & 291.25          & 320.37          & 385.99          & 512.73          & 596.39          & 13.39           & 14.24           & 15.82           & 18.57           & 20.71           \\
                       &                        & (14.98)         & (10.01)         & (41.84)         & (107.26)        & (46.08)         & (0.26)          & (0.42)          & (1.14)          & (2.54)          & (0.79)          \\ \hline
\multirow{6}{*}{Rightmost} & \multirow{2}{*}{MAE}   & 197.26          & 265.60          & 328.77          & 608.51          & 679.26          & 11.15           & 13.38           & 14.74           & 20.80           & 22.82           \\
                       &                        & (17.24)         & (27.79)         & (63.99)         & (76.85)         & (80.82)         & (0.62)          & (0.75)          & (1.55)          & (1.50)          & (2.21)          \\
                       & \multirow{2}{*}{MSE}   & 214.42          & 280.99          & 357.06          & 497.47          & 643.99          & 11.71           & 13.50           & 15.29           & 18.82           & 21.38           \\
                       &                        & (10.40)         & (12.52)         & (64.59)         & (43.98)         & (85.80)         & (0.35)          & (0.35)          & (1.54)          & (1.05)          & (1.49)          \\
                       & \multirow{2}{*}{Huber} & 198.97          & 243.81          & 379.34          & 536.87          & 491.39          & 11.23           & 12.82           & 16.19           & 19.63           & 18.24           \\
                       &                        & (6.47)          & (11.64)         & (45.10)         & (21.50)         & (103.54)        & (0.21)          & (0.31)          & (1.01)          & (0.67)          & (2.22)          \\ \hline
\multirow{6}{*}{Middlemost}   & \multirow{2}{*}{MAE}   & 140.29          & 148.93          & 152.20          & 154.79          & 157.17          & 8.96            & 9.24            & 9.47            & 9.55            & 9.76            \\
                       &                        & (6.68)          & (3.00)          & (9.59)          & (4.29)          & (5.96)          & (0.10)          & (0.08)          & (0.36)          & (0.21)          & (0.22)          \\
                       & \multirow{2}{*}{MSE}   & 135.57          & 142.68          & 144.34          & 153.10          & 153.80          & 8.89            & 9.10            & 9.23            & 9.61            & 9.79            \\
                       &                        & (2.77)          & (2.43)          & (2.60)          & (9.34)          & (3.57)          & (0.10)          & (0.06)          & (0.11)          & (0.30)          & (0.08)          \\
                       & \multirow{2}{*}{Huber} & 142.21          & 148.52          & 153.97          & 152.52          & 154.77          & 9.04            & 9.22            & 9.49            & 9.49            & 9.70            \\
                       &                        & (6.33)          & (3.93)          & (2.31)          & (3.48)          & (3.58)          & (0.19)          & (0.15)          & (0.05)          & (0.12)          & (0.15)          \\ \hline
\multicolumn{2}{c|}{\multirow{2}{*}{CRM}}       & 302.50          & 380.12          & 519.7           & 602.65          & 740.99          & 14.38           & 16.81           & 20.11           & 21.85           & 24.61           \\
\multicolumn{2}{c|}{}                           & (12.53)         & (15.30)         & (10.23)         & (7.20)          & (19.56)         & (0.42)          & (0.28)          & (0.20)          & (0.16)          & (0.78)          \\
\multicolumn{2}{c|}{\multirow{2}{*}{RANN}}      & 137.37          & 140.79          & 145.40          & 150.11          & 166.33          & 8.98            & 8.98            & 9.32            & 9.52            & 10.09           \\
\multicolumn{2}{c|}{}                           & (3.09)          & (2.28)          & (2.20)          & (2.48)          & (4.02)          & (0.10)          & (0.12)          & (0.08)          & (0.07)          & (0.08)          \\
\multicolumn{2}{c|}{\multirow{2}{*}{SINN}}      & 314.49          & 385.29          & 515.80          & 629.51          & 754.28          & 14.79           & 16.73           & 19.84           & 22.36           & 24.73           \\
\multicolumn{2}{c|}{}                           & (4.08)          & (8.26)          & (9.95)          & (11.10)         & (15.20)         & (0.14)          & (0.19)          & (0.30)          & (0.23)          & (0.65)          \\
\multicolumn{2}{c|}{\multirow{2}{*}{IN}}        & 147.09          & 148.71          & 152.66          & 155.73          & 156.04          & 9.25            & 9.28            & 9.51            & 9.59            & 9.79            \\
\multicolumn{2}{c|}{}                           & (0.59)          & (1.01)          & (2.54)          & (5.67)          & (2.89)          & (0.04)          & (0.09)          & (0.10)          & (0.17)          & (0.06)          \\ \hline
\multicolumn{2}{c|}{\multirow{2}{*}{LM}}        & \textbf{133.98} & \textbf{134.15} & \textbf{141.45} & \textbf{148.19} & \textbf{146.52} & \textbf{8.75}   & \textbf{8.83}   & \textbf{9.07}   & \textbf{9.40}   & \textbf{9.39}   \\
\multicolumn{2}{c|}{}                           & \textbf{(1.57)} & \textbf{(2.49)} & \textbf{(2.32)} & \textbf{(2.44)} & \textbf{(5.00)} & \textbf{(0.06)} & \textbf{(0.07)} & \textbf{(0.11)} & \textbf{(0.08)} & \textbf{(0.04)}
\\
\bottomrule
\end{tabular}
}
}
\end{table*}
\begin{figure*}[!t]
\centering
  \subfigure[Abalone]{
      \includegraphics[width=1.5in]{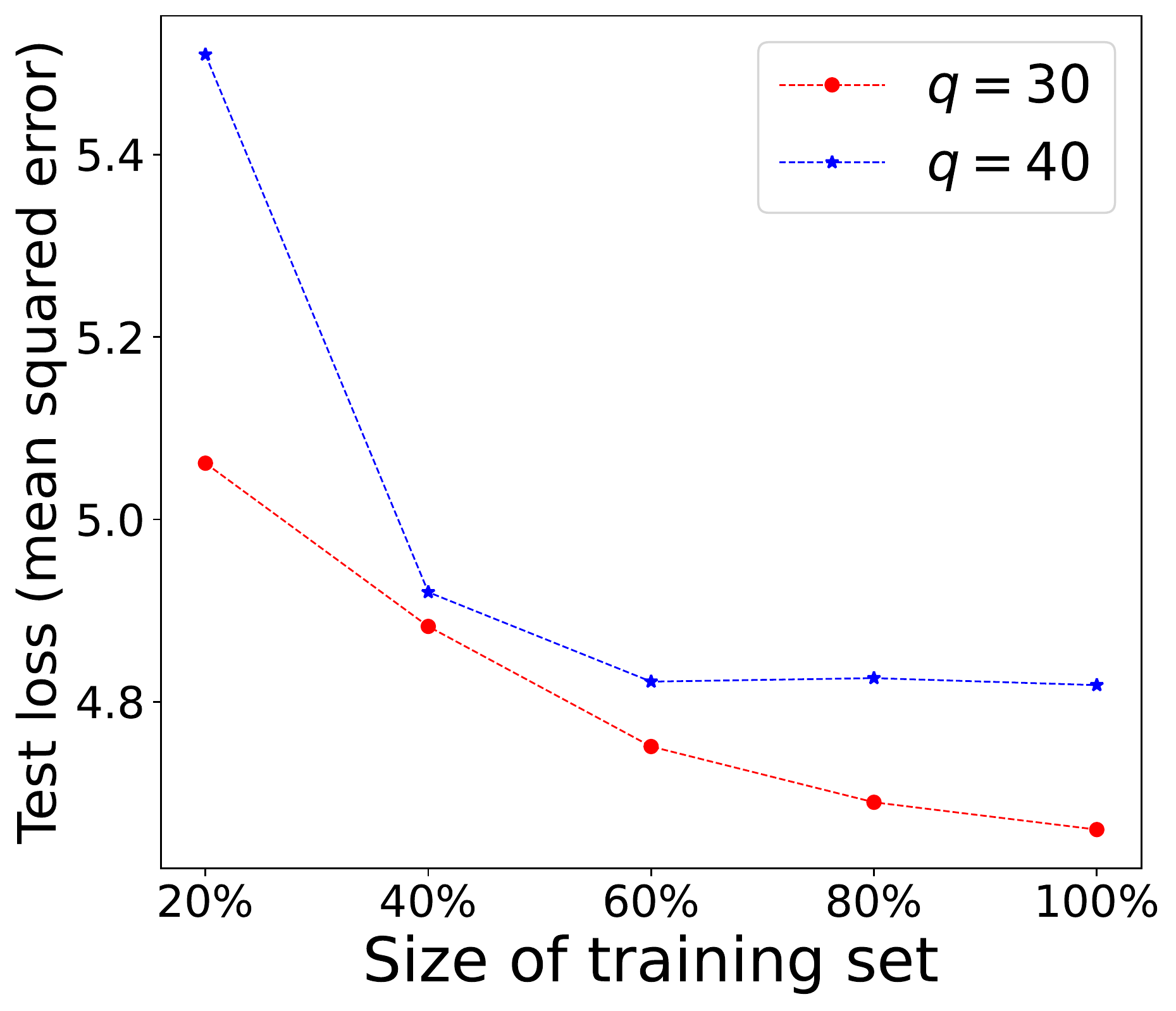}
      }
  \subfigure[Auto-mpg]{
      \includegraphics[width=1.5in]{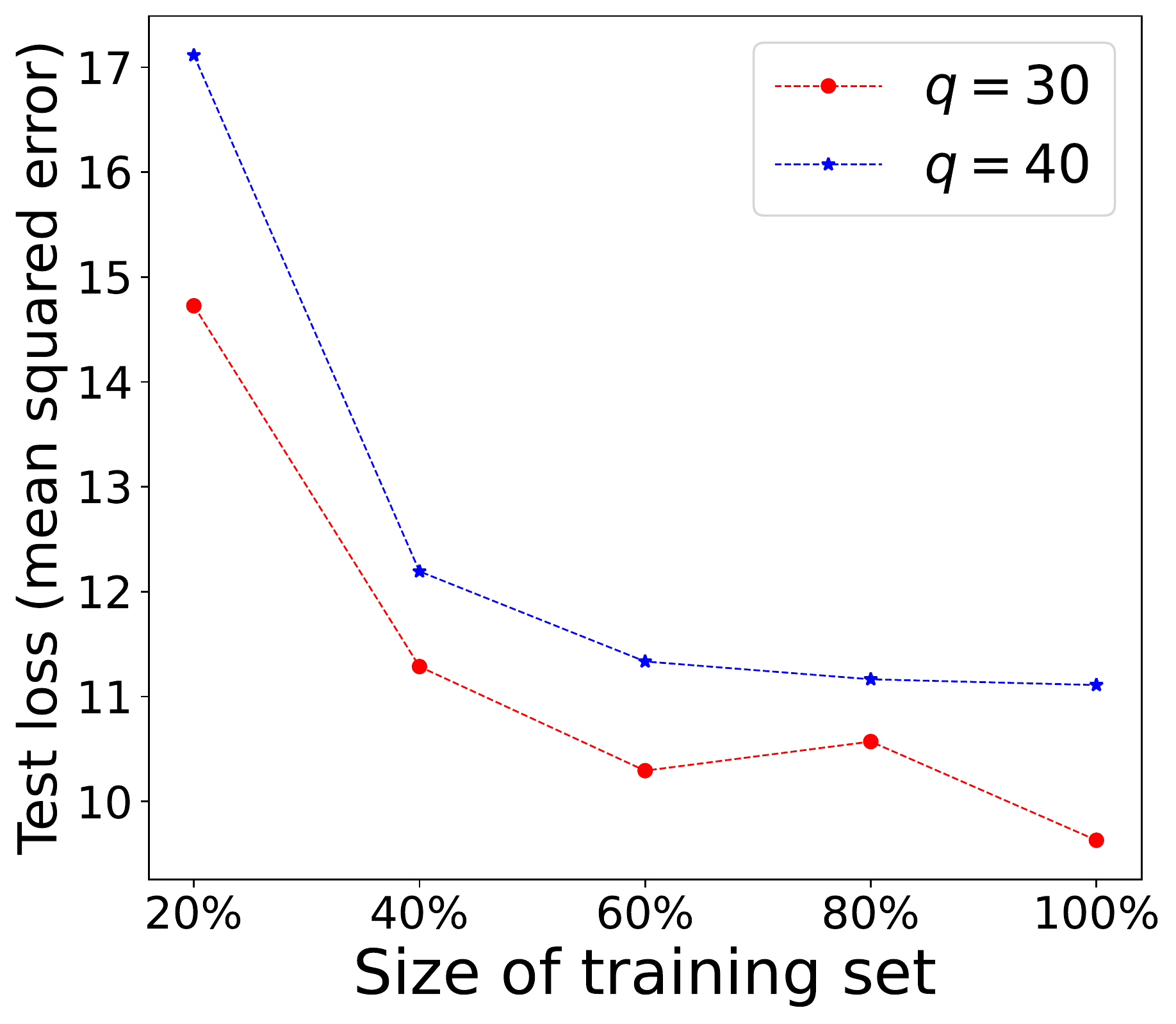}
      }
  \subfigure[Airfoil]{
      \includegraphics[width=1.5in]{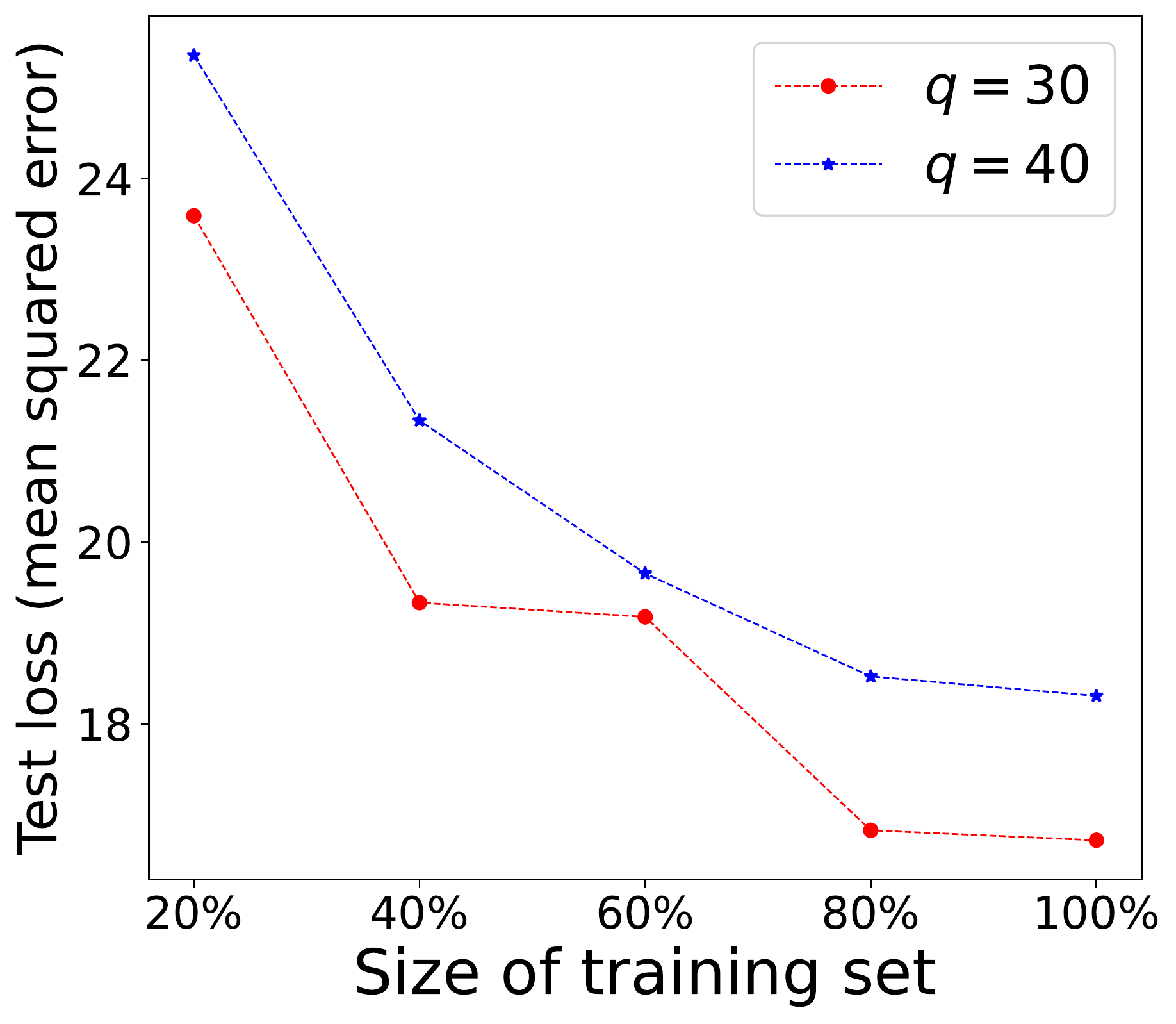}
      }
  \subfigure[Power-plant]{
      \includegraphics[width=1.5in]{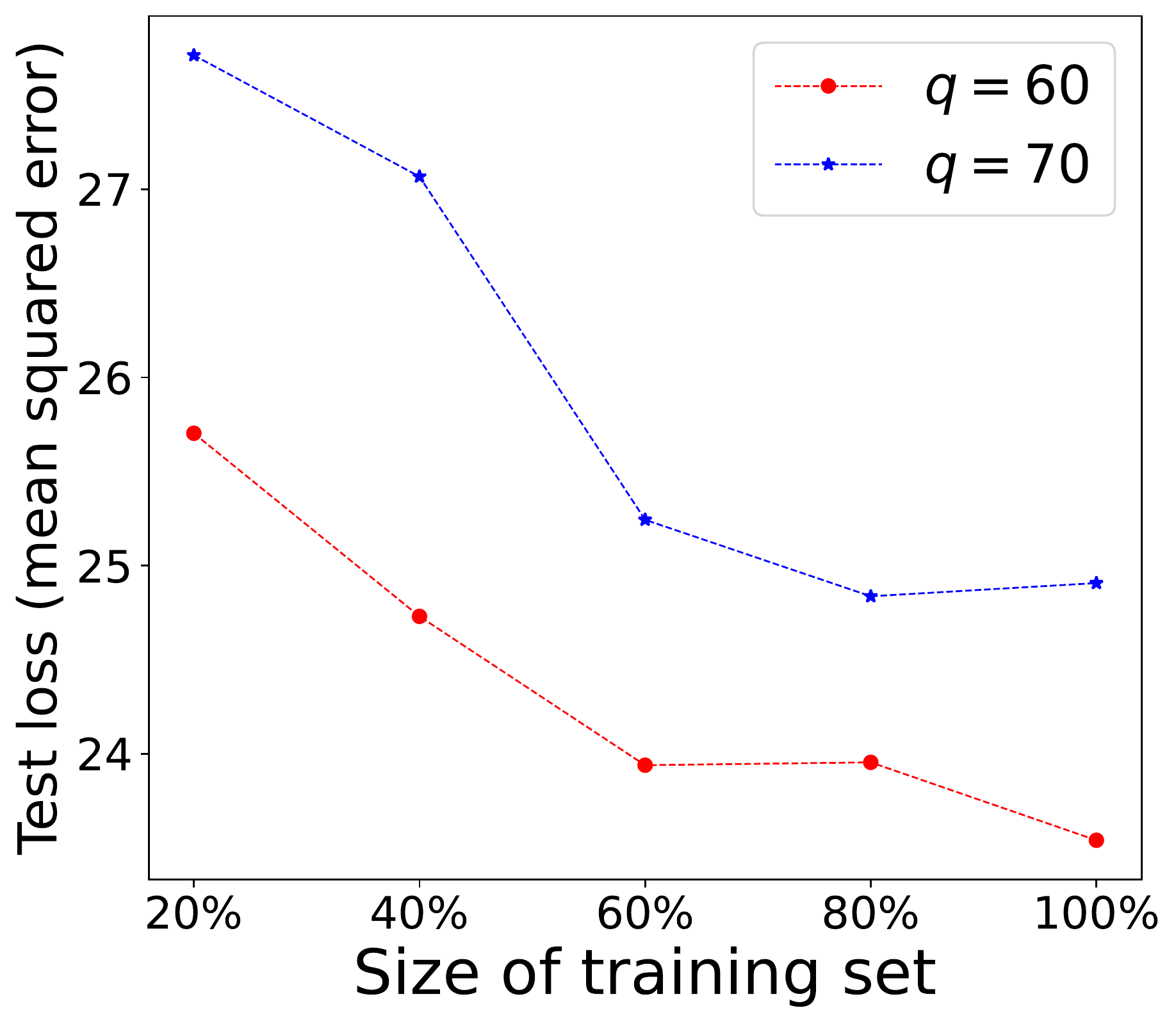}
      }
  \caption{The test performance (MSE) on the Abalone, Auto-mpg, Airfoil and Power-plant datasets of our proposed LM when the number of training data increases.}
  \label{fig2}
\vspace{-0.2cm}
\end{figure*}

\noindent\textbf{Evaluation metrics.}
For metrics, we use common evaluation metrics for regression, such as the MSE, MAE, and Pearson correlation. We also use another evaluation metric called Geometric Mean \cite{yang2021delving}.

\begin{table*}[!t]
\caption{Test performance (mean and std) of each method on the six UCI datasets trained with the MLP model. The used evaluation metrics include MSE and MAE. We repeat the sampling-and-training process 5 times. The best performance is highlighted in bold.}
\label{mlp}
\resizebox{0.99\textwidth}{!}{
\setlength{\tabcolsep}{1.9mm}{
\begin{tabular}{c|cc|ccc|ccc|ccc|cccc|c}
\toprule
\multirow{2}{*}{Dataset}    & \multirow{2}{*}{Metric} & \multirow{2}{*}{$q$}  & \multicolumn{3}{c|}{Leftmost}   & \multicolumn{3}{c|}{Rightmost}   & \multicolumn{3}{c|}{Middlemost}   & \multirow{2}{*}{CRM} & \multirow{2}{*}{RANN} & \multirow{2}{*}{SINN} & \multirow{2}{*}{IN} & \multirow{2}{*}{LM} \\
                             &                          &                     & MSE     & MAE     & Huber   & MSE      & MAE     & Huber   & MSE     & MAE    & Huber   &                      &                       &                       &                     &                     \\ \hline
\multirow{8}{*}{Abalone}     & \multirow{4}{*}{MSE}     & \multirow{2}{*}{30} & 65.33   & 55.73   & 59.82   & 7.99     & 8.08    & 8.03    & 6.38    & 5.45   & 6.01    & 6.51                 & 6.67                  & 6.44                  & 5.61                & \textbf{4.66}       \\
                             &                          &                     & (2.09)  & (1.79)  & (3.07)  & (0.70)   & (0.71)  & (0.71)  & (0.44)  & (0.43) & (0.44)  & (0.52)               & (0.48)                & (0.44)                & (0.40)              & \textbf{(0.49)}     \\
                             &                          & \multirow{2}{*}{40} & 90.12   & 85.71   & 89.27   & 8.13     & 8.23    & 8.21    & 7.77    & 7.81   & 7.79    & 7.85                 & 7.74                  & 7.85                  & 7.84                & \textbf{4.81}       \\
                             &                          &                     & (3.16)  & (6.28)  & (4.47)  & (0.67)   & (0.63)  & (0.64)  & (0.68)  & (0.76) & (0.76)  & (0.81)               & (0.64)                & (0.81)                & (0.75)              & \textbf{(0.42)}     \\ \cline{2-17} 
                             & \multirow{4}{*}{MAE}     & \multirow{2}{*}{30} & 7.60    & 6.96    & 7.21    & 2.01     & 2.11    & 2.10    & 1.94    & 1.77   & 1.91    & 1.94                 & 1.91                  & 1.94                  & 1.86                & \textbf{1.50}       \\
                             &                          &                     & (0.11)  & (0.13)  & (0.16)  & (0.08)   & (0.08)  & (0.08)  & (0.06)  & (0.08) & (0.07)  & (0.07)               & (0.07)                & (0.07)                & (0.05)              & \textbf{(0.05)}     \\
                             &                          & \multirow{2}{*}{40} & 9.03    & 8.80    & 9.03    & 2.02     & 2.12    & 2.11    & 1.96    & 1.95   & 1.96    & 1.99                 & 2.02                  & 1.99                  & 1.97                & \textbf{1.53}       \\
                             &                          &                     & (0.15)  & (0.32)  & (0.18)  & (0.08)   & (0.07)  & (0.07)  & (0.08)  & (0.06) & (0.06)  & (0.07)               & (0.08)                & (0.07)                & (0.06)              & \textbf{(0.08)}     \\ \hline
\multirow{8}{*}{Airfoil}     & \multirow{4}{*}{MSE}     & \multirow{2}{*}{30} & 122.58  & 89.68   & 93.90   & 105.20   & 79.28   & 85.16   & 19.24   & 18.71  & 18.24   & 17.40                & 17.30                 & 17.35                 & 19.20               & \textbf{16.72}      \\
                             &                          &                     & (89.68) & (5.26)  & (8.01)  & (9.19)   & (7.75)  & (14.44) & (1.70)  & (1.16) & (2.02)  & (3.07)               & (2.86)                & (3.13)                & (1.30)              & \textbf{(3.42)}     \\
                             &                          & \multirow{2}{*}{40} & 200.70  & 164.64  & 183.30  & 134.36   & 115.92  & 120.34  & 24.32   & 20.43  & 20.99   & 23.12                & 22.98                 & 23.08                 & 19.37               & \textbf{18.31}      \\
                             &                          &                     & (13.62) & (9.28)  & (23.42) & (13.66)  & (5.56)  & (7.57)  & (1.16)  & (1.87) & (1.89)  & (2.30)               & (1.92)                & (2.38)                & (2.51)              & \textbf{(2.63)}     \\ \cline{2-17} 
                             & \multirow{4}{*}{MAE}     & \multirow{2}{*}{30} & 9.81    & 8.50    & 8.71    & 8.00     & 7.78    & 7.98    & 3.41    & 3.30   & 3.28    & 3.26                 & 3.23                  & 3.26                  & 3.33                & \textbf{3.09}       \\
                             &                          &                     & (0.52)  & (0.39)  & (0.62)  & (0.26)   & (0.38)  & (0.51)  & (0.24)  & (0.15) & (0.25)  & (0.33)               & (0.32)                & (0.32)                & (0.16)              & \textbf{(0.35)}     \\
                             &                          & \multirow{2}{*}{40} & 11.93   & 11.96   & 12.18   & 8.93     & 9.24    & 9.29    & 3.97    & 3.54   & 3.61    & 3.85                 & 3.85                  & 3.85                  & 3.47                & \textbf{3.24}       \\
                             &                          &                     & (0.31)  & (0.38)  & (0.60)  & (0.64)   & (0.43)  & (0.46)  & (0.12)  & (0.13) & (0.20)  & (0.22)               & (0.22)                & (0.21)                & (0.26)              & \textbf{(0.27)}     \\ \hline
\multirow{8}{*}{Auto-mpg}    & \multirow{4}{*}{MSE}     & \multirow{2}{*}{30} & 76.60   & 34.99   & 50.43   & 51.52    & 30.41   & 24.06   & 11.59   & 11.69  & 11.52   & 11.68                & 23.60                 & 11.67                 & 11.60               & \textbf{9.63}       \\
                             &                          &                     & (12.72) & (12.83) & (18.98) & (20.31)  & (11.25) & (7.60)  & (1.42)  & (1.69) & (1.17)  & (1.25)               & (3.86)                & (1.28)                & (1.23)              & \textbf{(1.62)}     \\
                             &                          & \multirow{2}{*}{40} & 145.11  & 93.80   & 92.64   & 119.54   & 48.03   & 61.49   & 20.14   & 18.30  & 19.25   & 21.11                & 34.57                 & 21.08                 & 18.23               & \textbf{11.11}      \\
                             &                          &                     & (19.52) & (23.82) & (32.76) & (20.53)  & (15.25) & (28.88) & (5.44)  & (4.52) & (4.69)  & (6.18)               & (8.40)                & (6.30)                & (4.97)              & \textbf{(3.03)}     \\ \cline{2-17} 
                             & \multirow{4}{*}{MAE}     & \multirow{2}{*}{30} & 7.97    & 4.47    & 5.98    & 6.33     & 4.47    & 3.89    & 2.52    & 2.51   & 2.48    & 2.50                 & 3.76                  & 2.50                  & 2.50                & \textbf{2.19}       \\
                             &                          &                     & (0.76)  & (1.03)  & (1.34)  & (1.63)   & (0.94)  & (0.78)  & (0.23)  & (0.18) & (0.18)  & (0.19)               & (0.38)                & (0.19)                & (0.17)              & \textbf{(0.18)}     \\
                             &                          & \multirow{2}{*}{40} & 11.59   & 8.09    & 8.16    & 10.91    & 5.78    & 6.68    & 3.46    & 3.15   & 3.18    & 3.39                 & 4.62                  & 3.36                  & 3.24                & \textbf{2.31}       \\
                             &                          &                     & (1.02)  & (1.16)  & (1.84)  & (1.75)   & (1.03)  & (1.92)  & (0.48)  & (0.37) & (0.43)  & (0.50)               & (0.59)                & (0.52)                & (0.41)              & \textbf{(0.30)}     \\ \hline
\multirow{8}{*}{Housing}     & \multirow{4}{*}{MSE}     & \multirow{2}{*}{30} & 55.46   & 52.66   & 52.79   & 75.37    & 78.40   & 83.83   & 27.49   & 26.24  & 25.07   & 27.58                & 25.45                 & 26.32                 & 25.06               & \textbf{22.13}      \\
                             &                          &                     & (8.52)  & (15.47) & (7.06)  & (27.20)  & (8.87)  & (15.28) & (11.60) & (8.20) & (5.82)  & (6.80)               & (8.95)                & (10.12)               & (7.59)              & \textbf{(3.71)}     \\
                             &                          & \multirow{2}{*}{40} & 88.85   & 101.70  & 83.96   & 109.01   & 124.34  & 124.14  & 30.70   & 35.84  & 32.63   & 32.66                & 31.15                 & 34.92                 & 33.00               & \textbf{24.53}      \\
                             &                          &                     & (22.73) & (13.36) & (15.88) & (32.28)  & (12.16) & (12.18) & (4.04)  & (2.72) & (4.86)  & (4.47)               & (5.46)                & (4.47)                & (5.71)              & \textbf{8.00}       \\ \cline{2-17} 
                             & \multirow{4}{*}{MAE}     & \multirow{2}{*}{30} & 6.02    & 5.61    & 5.49    & 6.88     & 7.63    & 8.31    & 3.52    & 3.54   & 3.50    & 3.52                 & 3.51                  & 3.63                  & 3.56                & \textbf{3.47}       \\
                             &                          &                     & (0.65)  & (1.06)  & (0.61)  & (0.54)   & (0.33)  & (0.85)  & (0.36)  & (0.45) & (0.27)  & (0.50)               & (0.56)                & (0.66)                & (0.43)              & \textbf{(0.57)}     \\
                             &                          & \multirow{2}{*}{40} & 7.15    & 7.46    & 7.17    & 7.87     & 8.27    & 8.26    & 4.06    & 4.41   & 4.11    & 4.14                 & 3.95                  & 4.25                  & 4.19                & \textbf{3.49}       \\
                             &                          &                     & (0.66)  & (0.53)  & (0.62)  & (0.92)   & (0.42)  & (0.42)  & (0.45)  & (0.28) & (0.31)  & (0.45)               & (0.48)                & (0.39)                & (0.46)              & \textbf{(0.48)}     \\ \hline
\multirow{8}{*}{Concrete}    & \multirow{4}{*}{MSE}     & \multirow{2}{*}{70} & 283.72  & 291.78  & 296.00  & 278.35   & 278.30  & 278.26  & 70.94   & 72.84  & 73.01   & 74.68                & 74.16                 & 71.68                 & 67.22               & \textbf{58.45}      \\
                             &                          &                     & (13.93) & (17.85) & (7.37)  & (18.24)  & (18.16) & (18.18) & (8.19)  & (7.78) & (6.65)  & (5.80)               & (4.23)                & (5.84)                & (5.82)              & \textbf{(3.52)}     \\
                             &                          & \multirow{2}{*}{80} & 420.44  & 424.03  & 423.70  & 284.03   & 284.04  & 284.02  & 88.26   & 95.09  & 86.54   & 86.40                & 80.77                 & 83.38                 & 92.48               & \textbf{59.47}      \\
                             &                          &                     & (68.96) & (57.64) & (40.60) & (15.41)  & (15.41) & (15.43) & (6.35)  & (8.56) & (13.11) & (3.13)               & (2.78)                & (3.13)                & (16.38)             & \textbf{(7.81)}     \\ \cline{2-17} 
                             & \multirow{4}{*}{MAE}     & \multirow{2}{*}{70} & 13.93   & 14.07   & 14.42   & 13.38    & 13.38   & 13.38   & 6.70    & 6.81   & 6.83    & 6.78                 & 6.51                  & 6.77                  & 6.38                & \textbf{5.86}       \\
                             &                          &                     & (1.04)  & (0.99)  & (0.64)  & (0.49)   & (0.49)  & (0.49)  & (0.54)  & (0.25) & (0.32)  & (0.81)               & (0.82)                & (0.82)                & (0.34)              & \textbf{(0.41)}     \\
                             &                          & \multirow{2}{*}{80} & 17.74   & 17.67   & 17.93   & 13.47    & 13.47   & 13.47   & 7.42    & 7.62   & 7.35    & 7.40                 & 7.26                  & 7.48                  & 7.49                & \textbf{5.71}       \\
                             &                          &                     & (1.55)  & (1.48)  & (1.26)  & (0.38)   & (0.38)  & (0.38)  & (0.46)  & (0.89) & (0.61)  & (0.68)               & (0.33)                & (0.47)                & (0.70)              & \textbf{(0.34)}     \\ \hline
\multirow{8}{*}{Power-plant} & \multirow{4}{*}{MSE}     & \multirow{2}{*}{60} & 264.98  & 201.00  & 245.19  & 262.00   & 246.15  & 268.46  & 31.96   & 28.95  & 29.38   & 32.85                & 31.30                 & 32.66                 & 29.36               & \textbf{23.54}      \\
                             &                          &                     & (49.48) & (51.72) & (35.67) & (107.86) & (69.44) & (53.87) & (2.69)  & (3.03) & (3.76)  & (3.78)               & (2.84)                & (3.95)                & (3.41)              & \textbf{(1.41)}     \\
                             &                          & \multirow{2}{*}{70} & 279.48  & 316.33  & 337.53  & 280.29   & 269.03  & 293.51  & 48.82   & 39.41  & 40.83   & 48.53                & 47.62                 & 48.62                 & 40.04               & \textbf{24.91}      \\
                             &                          &                     & (3.14)  & (72.10) & (61.63) & (76.53)  & (81.76) & (2.20)  & (1.53)  & (2.76) & (3.41)  & (1.20)               & (2.37)                & (1.22)                & (3.13)              & \textbf{(0.60)}     \\ \cline{2-17} 
                             & \multirow{4}{*}{MAE}     & \multirow{2}{*}{60} & 14.48   & 12.55   & 14.08   & 13.92    & 13.98   & 14.60   & 4.55    & 4.30   & 4.36    & 4.59                 & 4.49                  & 4.58                  & 4.37                & \textbf{3.84}       \\
                             &                          &                     & (0.10)  & (1.58)  & (0.34)  & (3.31)   & (2.39)  & (1.43)  & (0.18)  & (0.21) & (0.32)  & (0.28)               & (0.21)                & (0.31)                & (0.29)              & \textbf{(0.11)}     \\
                             &                          & \multirow{2}{*}{70} & 14.29   & 14.48   & 14.70   & 14.17    & 14.70   & 14.99   & 5.59    & 4.94   & 5.13    & 5.62                 & 5.51                  & 5.62                  & 5.01                & \textbf{3.96}       \\
                             &                          &                     & (0.16)  & (1.51)  & (0.94)  & (1.85)   & (2.56)  & (0.10)  & (0.14)  & (0.18) & (0.20)  & (0.08)               & (0.10)                & (0.07)                & (0.16)              & \textbf{(0.05)}    
\\
\bottomrule
\end{tabular}
}
}
\end{table*}
\subsection{Experimental Performance}
\noindent\textbf{Experimental results.}
Table \ref{agedb}, Table \ref{imdb} and Table \ref{mlp} show some of the experimental results on the AgeDB, IMDB-WIKI, and UCI datasets, respectively. From the three tables, we have the following observations:
1) Our proposed LM outperforms all the compared methods. This verifies that our method has the ability to figure out the true real-valued labels.
2) As $q$ increases, there is a tendency for the performance of all the methods to decrease. This is because as the size of the interval becomes larger, more interfering values are included in the interval and thus it will be more difficult to identify the true real-valued labels from the interval.
3) In our experimental setting, we set various values of $q$. The performance gap between our method and compared methods is more evident when $q$ is large. This indicates that our method has stronger robustness. It is worth noting that $q$ represents the maximum interval size allowed. In real-world scenarios, a large value of $q$ will be a more common situation because this kind of data is easier to collect.
4) The methods for interval-valued data prediction and the methods for selecting the middlemost value of the interval as the target value have similar performance. This is because when the interval predicted by the interval-valued data prediction is accurate, the middlemost value of the interval is exactly the target value of the methods that select the middlemost value of the interval as the target value.

\noindent\textbf{Performance of Increasing Training Data.}
We demonstrate in Theorem \ref{error_bound} that the model learned by our proposed LM can converge to the optimal model learned from the fully labeled data when the number of training examples for RIT approaches infinity. To empirically validate such a theoretical finding, we further conduct experiments by changing the fraction of training examples for RIT, where 100\% indicates the use of all training examples to train the model. The experimental performance of LM is shown in Figure \ref{fig2}, where the test loss of the model generally decreases when more training examples are used to train the model. This empirical observation accords with our theoretical analysis that the learned model will be closer to the optimal model, if more training examples are provided.

\noindent\textbf{More experimental results.} We provide more experimental results including the comparison results with the fully supervised method, maximum margin interval trees method (MMIT \cite{drouin2017maximum}), and more results on evaluation metrics and models in Appendix \ref{all_result}. These results also demonstrate the effectiveness of our method.
\section{Conclusion}
In this paper, we studied an interesting weakly supervised regression setting called \emph{regression with interval targets} (RIT). For the RIT setting, we first proposed a novel statistical model to describe the data generation process for RIT and demonstrated its validity. The explicitly derived data distribution can be helpful to empirical risk minimization. Then, we analyzed a simple selection method that selects a particular value in the interval as the target value to train the model. We empirically showed that this simple method could work well if the middlemost value in the interval is selected. Afterward, we proposed a statistically consistent limiting method to train the model by limiting the predictions to the interval. We further derived an estimation error bound for this method. Finally, we conducted extensive experiments on various datasets to demonstrate the effectiveness of our proposed method. In future work, it would be interesting to study a harder setting of RIT, where the true target value might be outside the given interval.

\section*{Acknowledgements}
This research is supported, in part, by the Joint NTU-WeBank Research Centre on Fintech (Award No: NWJ-2021-005), Nanyang Technological University, Singapore. Lei Feng is also supported by the National Natural Science Foundation of China (Grant No. 62106028), Chongqing Overseas Chinese Entrepreneurship and Innovation Support Program, CAAI-Huawei MindSpore Open Fund, and Chongqing Artificial Intelligence Innovation Center. Ximing Li is supported by the National Key R\&D Program of China (No. 2021ZD0112501, No. 2021ZD0112502) and the National Natural Science Foundation of China (No. 62276113).

\bibliography{main}

\begin{thebibliography}{49}
\providecommand{\natexlab}[1]{#1}
\providecommand{\url}[1]{\texttt{#1}}
\expandafter\ifx\csname urlstyle\endcsname\relax
  \providecommand{\doi}[1]{doi: #1}\else
  \providecommand{\doi}{doi: \begingroup \urlstyle{rm}\Url}\fi

\bibitem[Amar et~al.(2001)Amar, Dooly, Goldman, and Zhang]{amar2001multiple}
Amar, R.~A., Dooly, D.~R., Goldman, S.~A., and Zhang, Q.
\newblock Multiple-instance learning of real-valued data.
\newblock In \emph{ICML}, pp.\  3--10, 2001.

\bibitem[Billard(2006)]{billard2006symbolic}
Billard, L.
\newblock Symbolic data analysis: what is it?
\newblock In \emph{Compstat 2006-Proceedings in Computational Statistics: 17th
  Symposium Held in Rome, Italy, 2006}, pp.\  261--269. Springer, 2006.

\bibitem[Billard \& Diday(2000)Billard and Diday]{billard2000regression}
Billard, L. and Diday, E.
\newblock Regression analysis for interval-valued data.
\newblock In \emph{Data Analysis, Classification, and Related Methods}, pp.\
  369--374. Springer, 2000.

\bibitem[Bock \& Diday(1999)Bock and Diday]{bock1999analysis}
Bock, H.-H. and Diday, E.
\newblock \emph{Analysis of symbolic data: exploratory methods for extracting
  statistical information from complex data}.
\newblock Springer Science \& Business Media, 1999.

\bibitem[Carpentier \& Schl{\"u}ter(2016)Carpentier and
  Schl{\"u}ter]{carpentier2016learning}
Carpentier, A. and Schl{\"u}ter, T.
\newblock Learning relationships between data obtained independently.
\newblock In \emph{AISTATS}, pp.\  658--666, 2016.

\bibitem[Cer et~al.(2017)Cer, Diab, Agirre, Lopez-Gazpio, and
  Specia]{cer2017semeval}
Cer, D., Diab, M., Agirre, E., Lopez-Gazpio, I., and Specia, L.
\newblock Semeval-2017 task 1: Semantic textual similarity-multilingual and
  cross-lingual focused evaluation.
\newblock \emph{arXiv preprint arXiv:1708.00055}, 2017.

\bibitem[Cour et~al.(2011)Cour, Sapp, and Taskar]{cour2011learning}
Cour, T., Sapp, B., and Taskar, B.
\newblock Learning from partial labels.
\newblock \emph{The Journal of Machine Learning Research}, 12:\penalty0
  1501--1536, 2011.

\bibitem[Drouin et~al.(2017)Drouin, Hocking, and Laviolette]{drouin2017maximum}
Drouin, A., Hocking, T., and Laviolette, F.
\newblock Maximum margin interval trees.
\newblock In \emph{NeurIPS}, 2017.

\bibitem[Dua \& Graff(2017)Dua and Graff]{Dua2019UCI}
Dua, D. and Graff, C.
\newblock {UCI} machine learning repository, 2017.
\newblock URL \url{http://archive.ics.uci.edu/ml}.

\bibitem[Fagundes et~al.(2014)Fagundes, De~Souza, and
  Cysneiros]{fagundes2014interval}
Fagundes, R.~A., De~Souza, R.~M., and Cysneiros, F. J.~A.
\newblock Interval kernel regression.
\newblock \emph{Neurocomputing}, 128:\penalty0 371--388, 2014.

\bibitem[Feng et~al.(2020)Feng, Lv, Han, Xu, Niu, Geng, An, and
  Sugiyama]{feng2020provably}
Feng, L., Lv, J., Han, B., Xu, M., Niu, G., Geng, X., An, B., and Sugiyama, M.
\newblock Provably consistent partial-label learning.
\newblock In \emph{NeurIPS}, pp.\  10948--10960, 2020.

\bibitem[Geng et~al.(2013)Geng, Yin, and Zhou]{geng2013facial}
Geng, X., Yin, C., and Zhou, Z.-H.
\newblock Facial age estimation by learning from label distributions.
\newblock \emph{IEEE Transactions on Pattern Analysis and Machine
  Intelligence}, 35\penalty0 (10):\penalty0 2401--2412, 2013.

\bibitem[Giordani(2015)]{giordani2015lasso}
Giordani, P.
\newblock Lasso-constrained regression analysis for interval-valued data.
\newblock \emph{Advances in Data Analysis and Classification}, 9\penalty0
  (1):\penalty0 5--19, 2015.

\bibitem[He et~al.(2016)He, Zhang, Ren, and Sun]{he2016deep}
He, K., Zhang, X., Ren, S., and Sun, J.
\newblock Deep residual learning for image recognition.
\newblock In \emph{CVPR}, pp.\  770--778, 2016.

\bibitem[Hu et~al.(2020)Hu, Li, and Yu]{hu2020simple}
Hu, W., Li, Z., and Yu, D.
\newblock Simple and effective regularization methods for training on noisily
  labeled data with generalization guarantee.
\newblock In \emph{ICLR}, 2020.

\bibitem[Ishibuchi \& Tanaka(1991)Ishibuchi and Tanaka]{ishibuchi1991extension}
Ishibuchi, H. and Tanaka, H.
\newblock An extension of the bp-algorithm to interval input vectors-learning
  from numerical data and expert's knowledge.
\newblock In \emph{IJCNN}, pp.\  1588--1593, 1991.

\bibitem[Ishida et~al.(2019)Ishida, Niu, Menon, and
  Sugiyama]{ishida2019complementary}
Ishida, T., Niu, G., Menon, A., and Sugiyama, M.
\newblock Complementary-label learning for arbitrary losses and models.
\newblock In \emph{ICML}, pp.\  2971--2980, 2019.

\bibitem[Kingma \& Ba(2015)Kingma and Ba]{kingma2014adam}
Kingma, D.~P. and Ba, J.
\newblock Adam: {A} method for stochastic optimization.
\newblock In \emph{ICLR}, 2015.

\bibitem[Kleinbaum et~al.(2012)Kleinbaum, Klein, et~al.]{kleinbaum2012survival}
Kleinbaum, D.~G., Klein, M., et~al.
\newblock \emph{Survival analysis: a self-learning text}, volume~3.
\newblock Springer, 2012.

\bibitem[Kostopoulos et~al.(2018)Kostopoulos, Karlos, Kotsiantis, and
  Ragos]{kostopoulos2018semi}
Kostopoulos, G., Karlos, S., Kotsiantis, S., and Ragos, O.
\newblock Semi-supervised regression: A recent review.
\newblock \emph{Journal of Intelligent \& Fuzzy Systems}, 35\penalty0
  (2):\penalty0 1483--1500, 2018.

\bibitem[Lauro \& Palumbo(2000)Lauro and Palumbo]{lauro2000principal}
Lauro, C.~N. and Palumbo, F.
\newblock Principal component analysis of interval data: a symbolic data
  analysis approach.
\newblock \emph{Computational Statistics}, 15\penalty0 (1):\penalty0 73--87,
  2000.

\bibitem[Lesaffre et~al.(2005)Lesaffre, Kom{\'a}rek, and
  Declerck]{lesaffre2005overview}
Lesaffre, E., Kom{\'a}rek, A., and Declerck, D.
\newblock An overview of methods for interval-censored data with an emphasis on
  applications in dentistry.
\newblock \emph{Statistical Methods in Medical Research}, 14\penalty0
  (6):\penalty0 539--552, 2005.

\bibitem[Li et~al.(2017)Li, Zha, and Zhou]{li2017learning}
Li, Y.-F., Zha, H.-W., and Zhou, Z.-H.
\newblock Learning safe prediction for semi-supervised regression.
\newblock In \emph{AAAI}, 2017.

\bibitem[Lindsey \& Ryan(1998)Lindsey and Ryan]{lindsey1998methods}
Lindsey, J.~C. and Ryan, L.~M.
\newblock Methods for interval-censored data.
\newblock \emph{Statistics in Medicine}, 17\penalty0 (2):\penalty0 219--238,
  1998.

\bibitem[Lv et~al.(2020)Lv, Xu, Feng, Niu, Geng, and
  Sugiyama]{lv2020progressive}
Lv, J., Xu, M., Feng, L., Niu, G., Geng, X., and Sugiyama, M.
\newblock Progressive identification of true labels for partial-label learning.
\newblock In \emph{ICML}, 2020.

\bibitem[Machin et~al.(2006)Machin, Cheung, and Parmar]{machin2006survival}
Machin, D., Cheung, Y.~B., and Parmar, M.
\newblock \emph{Survival analysis: a practical approach}.
\newblock John Wiley \& Sons, 2006.

\bibitem[Manski \& Tamer(2002)Manski and Tamer]{manski2002inference}
Manski, C.~F. and Tamer, E.
\newblock Inference on regressions with interval data on a regressor or
  outcome.
\newblock \emph{Econometrica}, 70\penalty0 (2):\penalty0 519--546, 2002.

\bibitem[Mohri et~al.(2012)Mohri, Rostamizadeh, and
  Talwalkar]{mohri2012foundations}
Mohri, M., Rostamizadeh, A., and Talwalkar, A.
\newblock \emph{Foundations of Machine Learning}.
\newblock MIT Press, 2012.

\bibitem[Moschoglou et~al.(2017)Moschoglou, Papaioannou, Sagonas, Deng, Kotsia,
  and Zafeiriou]{8014984}
Moschoglou, S., Papaioannou, A., Sagonas, C., Deng, J., Kotsia, I., and
  Zafeiriou, S.
\newblock Agedb: The first manually collected, in-the-wild age database.
\newblock In \emph{CVPRW}, pp.\  1997--2005, 2017.

\bibitem[Neto \& De~Carvalho(2008)Neto and De~Carvalho]{neto2008centre}
Neto, E. d. A.~L. and De~Carvalho, F. D.~A.
\newblock Centre and range method for fitting a linear regression model to
  symbolic interval data.
\newblock \emph{Computational Statistics \& Data Analysis}, 52\penalty0
  (3):\penalty0 1500--1515, 2008.

\bibitem[Neto \& De~Carvalho(2010)Neto and De~Carvalho]{neto2010constrained}
Neto, E. d. A.~L. and De~Carvalho, F. D.~A.
\newblock Constrained linear regression models for symbolic interval-valued
  variables.
\newblock \emph{Computational Statistics \& Data Analysis}, 54\penalty0
  (2):\penalty0 333--347, 2010.

\bibitem[Park et~al.(2020)Park, Wang, Lim, Xiao, Lu, and
  Wang]{park2020bayesian}
Park, S., Wang, X., Lim, J., Xiao, G., Lu, T., and Wang, T.
\newblock Bayesian multiple instance regression for modeling immunogenic
  neoantigens.
\newblock \emph{Statistical Methods in Medical Research}, 29\penalty0
  (10):\penalty0 3032--3047, 2020.

\bibitem[Patrini et~al.(2017)Patrini, Rozza, Krishna~Menon, Nock, and
  Qu]{patrini2017making}
Patrini, G., Rozza, A., Krishna~Menon, A., Nock, R., and Qu, L.
\newblock Making deep neural networks robust to label noise: A loss correction
  approach.
\newblock In \emph{CVPR}, pp.\  1944--1952, 2017.

\bibitem[Rabinowitz et~al.(1995)Rabinowitz, Tsiatis, and
  Aragon]{rabinowitz1995regression}
Rabinowitz, D., Tsiatis, A., and Aragon, J.
\newblock Regression with interval-censored data.
\newblock \emph{Biometrika}, 82\penalty0 (3):\penalty0 501--513, 1995.

\bibitem[Ristovski et~al.(2010)Ristovski, Das, Ouzienko, Guo, and
  Obradovic]{ristovski2010regression}
Ristovski, K., Das, D., Ouzienko, V., Guo, Y., and Obradovic, Z.
\newblock Regression learning with multiple noisy oracles.
\newblock In \emph{ECAI}, pp.\  445--450, 2010.

\bibitem[Rothe et~al.(2018)Rothe, Timofte, and Van~Gool]{rothe2018deep}
Rothe, R., Timofte, R., and Van~Gool, L.
\newblock Deep expectation of real and apparent age from a single image without
  facial landmarks.
\newblock \emph{International Journal of Computer Vision}, 126\penalty0
  (2):\penalty0 144--157, 2018.

\bibitem[Sadeghi et~al.(2019)Sadeghi, De~Angelis, and
  Patelli]{sadeghi2019efficient}
Sadeghi, J., De~Angelis, M., and Patelli, E.
\newblock Efficient training of interval neural networks for imprecise training
  data.
\newblock \emph{Neural Networks}, 118:\penalty0 338--351, 2019.

\bibitem[Stulp \& Sigaud(2015)Stulp and Sigaud]{stulp2015many}
Stulp, F. and Sigaud, O.
\newblock Many regression algorithms, one unified model: A review.
\newblock \emph{Neural Networks}, 69:\penalty0 60--79, 2015.

\bibitem[Sun(2006)]{sun2006statistical}
Sun, J.
\newblock \emph{The statistical analysis of interval-censored failure time
  data}.
\newblock Springer, 2006.

\bibitem[Uysal \& G{\"u}venir(1999)Uysal and G{\"u}venir]{uysal1999overview}
Uysal, I. and G{\"u}venir, H.~A.
\newblock An overview of regression techniques for knowledge discovery.
\newblock \emph{The Knowledge Engineering Review}, 14\penalty0 (4):\penalty0
  319--340, 1999.

\bibitem[Wang et~al.(2019{\natexlab{a}})Wang, Singh, Michael, Hill, Levy, and
  Bowman]{wang2019glue}
Wang, A., Singh, A., Michael, J., Hill, F., Levy, O., and Bowman, S.~R.
\newblock Glue: A multi-task benchmark and analysis platform for natural
  language understanding.
\newblock In \emph{ICLR}, 2019{\natexlab{a}}.

\bibitem[Wang et~al.(2019{\natexlab{b}})Wang, Li, and Reddy]{wang2019machine}
Wang, P., Li, Y., and Reddy, C.~K.
\newblock Machine learning for survival analysis: A survey.
\newblock \emph{ACM Computing Surveys}, 51\penalty0 (6):\penalty0 1--36,
  2019{\natexlab{b}}.

\bibitem[Wang et~al.(2011)Wang, Lan, and Vucetic]{wang2011mixture}
Wang, Z., Lan, L., and Vucetic, S.
\newblock Mixture model for multiple instance regression and applications in
  remote sensing.
\newblock \emph{IEEE Transactions on Geoscience and Remote Sensing},
  50\penalty0 (6):\penalty0 2226--2237, 2011.

\bibitem[Wasserman \& Lafferty(2007)Wasserman and
  Lafferty]{wasserman2007statistical}
Wasserman, L. and Lafferty, J.
\newblock Statistical analysis of semi-supervised regression.
\newblock In \emph{NeurIPS}, 2007.

\bibitem[Xu et~al.(2019)Xu, Honda, Niu, and Sugiyama]{xu2019uncoupled}
Xu, L., Honda, J., Niu, G., and Sugiyama, M.
\newblock Uncoupled regression from pairwise comparison data.
\newblock In \emph{NeurIPS}, 2019.

\bibitem[Yang \& Liu(2018)Yang and Liu]{yang2018l1}
Yang, D. and Liu, Y.
\newblock L1/2 regularization learning for smoothing interval neural networks:
  Algorithms and convergence analysis.
\newblock \emph{Neurocomputing}, 272:\penalty0 122--129, 2018.

\bibitem[Yang \& Wu(2012)Yang and Wu]{yang2012smoothing}
Yang, D. and Wu, W.
\newblock A smoothing interval neural network.
\newblock \emph{Discrete Dynamics in Nature and Society}, 2012, 2012.

\bibitem[Yang et~al.(2021)Yang, Zha, Chen, Wang, and Katabi]{yang2021delving}
Yang, Y., Zha, K., Chen, Y., Wang, H., and Katabi, D.
\newblock Delving into deep imbalanced regression.
\newblock In \emph{ICML}, pp.\  11842--11851, 2021.

\bibitem[Yang et~al.(2019)Yang, Lin, and Zhang]{yang2019interval}
Yang, Z., Lin, D.~K., and Zhang, A.
\newblock Interval-valued data prediction via regularized artificial neural
  network.
\newblock \emph{Neurocomputing}, 331:\penalty0 336--345, 2019.

\end{thebibliography}
\bibliographystyle{icml2023}

\newpage
\appendix
\onecolumn
\section{Proofs about the Problem Setting}
\subsection{Prove of Theorem \ref{thm1}}
\label{A.1}
For a specific label $y$, we define the set of all the possible intervals whose size is less than $q$ as
\begin{equation}
\nonumber
\mathcal{S}_{q}^{y}=\{S\mid S\in \mathcal{S},|S|\leq q,y\in S \}
\end{equation}

Since $\mathcal{S}_{q}^{y}$ is a continuous space, we use the sum of the number of all possible intervals to represent the size of $\mathcal{S}_{q}^{y}$, that is, the integral over all possible intervals. We might as well discuss the range of values of $\underline{y}$ and then fix $\underline{y}$ to discuss the values of $\overline{y}$. We can easily know that $\underline{y} \in [y-q,y]$. If $\underline{y} < y-q $, even the largest interval $S^\prime = [\underline{y},\underline{y}+q]$ cannot contain $y$ ($\underline{y}+q<y$). If $\underline{y} > y$, then the interval must not contain $y$ ($\underline{y} > y$). After determining $\underline{y}$, the maximum value of $\overline{y}$ is $\underline{y}+q$ ($\overline{y}-\underline{y} \leq q$) and the minimum value is $y$ ($y\leq \overline{y}$), so $\overline{y}\in[y,\underline{y}+q]$, then $|\mathcal{S}_{q}^{y}|=\int_{y-q}^{y}\int_{y}^{\underline{y}+q}1\ \mathrm{d}\overline{y}\mathrm{d}\underline{y}=\frac{q^2}{2}$. From our formulation of the interval data distribution $\tilde{p}(\bm{x},S)$, we can obtain the simplified expression $\tilde{p}(\bm{x},S)=\frac{2}{q^2}\int_{y \in S}p(\bm{x},y)\ \mathrm{d}y$. Then, we have
\begin{align}
\nonumber
\int_{\mathcal{S}}\int_{\mathcal{X}}\tilde{p}(\bm{x},S)\mathrm{d}\bm{x}\ \mathrm{d}S &= \int_{\mathcal{X}}\int_{\mathcal{S}}\int_{y\in \mathcal{S}} \frac{2}{q^2}p(\bm{x},y)\ \mathrm{d}y\mathrm{d}S\mathrm{d}\bm{x}\\
\nonumber
&=\frac{2}{q^2}\int_{\mathcal{X}}\int_{\mathcal{Y}}\int_{\mathcal{S}_{q}^{y}} p(\bm{x},y)\ \mathrm{d}S\mathrm{d}y\mathrm{d}\bm{x}\\
\nonumber
&=\frac{2}{q^2}\int_{\mathcal{X}}\int_{\mathcal{Y}}p(\bm{x},y)\int_{\mathcal{S}_{q}^{y}}1\ \mathrm{d}S\mathrm{d}y\mathrm{d}\bm{x}\\
\nonumber
&=\frac{2}{q^2}\int_{\mathcal{X}}\int_{\mathcal{Y}}p(\bm{x},y)|\mathcal{S}_{q}^{y}|\mathrm{d}y\mathrm{d}\bm{x}\\
\nonumber
&=\frac{2}{q^2}\int_{\mathcal{X}}\int_{\mathcal{Y}}p(\bm{x},y)\frac{q^2}{2}\mathrm{d}y\mathrm{d}\bm{x}\\
\nonumber
&=\int_{\mathcal{X}}\int_{\mathcal{Y}}p(\bm{x},y)\mathrm{d}y\mathrm{d}\bm{x}\\
\nonumber
&=1,
\end{align}
which concludes the proof of Theorem \ref{thm1}.\qed

\subsection{Prove of Theorem \ref{key}}
\label{key_}
It is intuitive to express $\mathrm{Pr}[y\in S|\bm{x},S]$ as
\begin{align}
\nonumber
\mathrm{Pr}[y\in S|\bm{x},S]&= 1-\mathrm{Pr}[y\notin S|\bm{x},S]\\
\nonumber
&= 1-\int_{y\notin S}p(y|\bm{x},S)\mathrm{d}y\\
\nonumber
&=1-\int_{y\notin S} \frac{p(S|y,\bm{x})p(y,\bm{x})}{p(S|\bm{x})}\mathrm{d}y\\
\nonumber
&=1-\int_{y\notin S} \frac{p(S|y)p(y,\bm{x})}{\int_{y'\in S}p(S|y')p(y'|\bm{x})\mathrm{d}y'}\mathrm{d}y\\
\nonumber
&=1-|\mathcal{S}_{q}^{y'}|\int_{y\notin S} \frac{p(S|y)p(y,\bm{x})}{\int_{y\in S}p(y'|\bm{x})\mathrm{d}y'}\mathrm{d}y\\
\nonumber
&=1-\frac{q^2}{2}\int_{y\notin S} \frac{p(S|y)p(y,\bm{x})}{\int_{y'\in S}p(y'|\bm{x})\mathrm{d}y'}\mathrm{d}y\\
\nonumber
&=1,
\end{align}
where the last equality holds because $p(S|y) = 0$ if $y \notin S$ , in terms of Eq.~(\ref{equ3}). Which concludes the proof of Theorem \ref{key}.\qed
\subsection{Prove of Lemma \ref{lem1}}
\label{A.2}
We consider the case where the true label $y$ is a specific value, then we have
\begin{align}
\nonumber
p(y\in S,y|\bm{x})&= \mathrm{Pr}[y\in S|\bm{x},y]p(y|\bm{x})\\
\nonumber
&=\int_{\mathcal{S}}p(y\in S,S|\bm{x},y)p(y|\bm{x})\ \mathrm{d}S\\
\nonumber
&=\int_{\mathcal{S}}\mathrm{Pr}[y\in S|\bm{x},y,S]p(y|\bm{x})p(S|\bm{x},y)\ \mathrm{d}S\\
\nonumber
&=\int_{\mathcal{S}}\mathrm{Pr}[y\in S|\bm{x},y,S]p(y|\bm{x})p(S)\ \mathrm{d}S
\end{align}

where the last equality holds due to the fact that for each example $(\bm{x},y)$, $S$ is uniformly and randomly chosen, if $q$ is specific, $p(S)=\frac{1}{|\mathcal{S}_{q}|}$. As with $\mathcal{S}_{q}^{y}$, we integrate over all possible intervals to calculate the size of $|\mathcal{S}_{q}|$. We can easily know that when $\underline{y} \in [y_{\mathrm{min}}-q,y_{\mathrm{min}}]$, $\overline{y} \in [y_{\mathrm{min}},\underline{y}+q]$, when $\underline{y} \in [y_{\mathrm{min}},y_{\mathrm{max}}]$, $\overline{y} \in [\underline{y},\underline{y}+q]$, so $|S_{q}|= \int_{y_{\mathrm{min}}-q}^{y_{\mathrm{min}}}\int_{y_{\mathrm{min}}}^{\underline{y}+q}1\ \mathrm{d} y_{r}\mathrm{d} \underline{y}+ \int_{y_{\mathrm{min}}}^{y_{\mathrm{max}}}\int_{\underline{y}}^{\underline{y}+q}1\ \mathrm{d} \overline{y}\mathrm{d} \underline{y} = \frac{1}{2}q^2 + q(y_{\mathrm{max}}-y_{\mathrm{min}}) $, we have
\begin{align}
\nonumber
p(y\in S,y|\bm{x})&=\int_{\mathcal{S}}\mathrm{Pr}[y\in S|\bm{x},y,S]p(y|\bm{x})p(S)\ \mathrm{d}S\\
\nonumber
&=\frac{2}{2q(y_{\mathrm{max}}-y_{\mathrm{min}})+q^2}\int_{\mathcal{S}}\mathrm{Pr}[y\in S|\bm{x},y,S]p(y|\bm{x})\ \mathrm{d}S\\
\nonumber
&=\frac{2}{2q(y_{\mathrm{max}}-y_{\mathrm{min}})+q^2}\int_{\mathcal{S}}\mathrm{Pr}[y\in S|\bm{x},y,S]\mathrm{d}S\ p(y|\bm{x})\\
\nonumber
&=\frac{2}{2q(y_{\mathrm{max}}-y_{\mathrm{min}})+q^2}|\mathcal{S}_{q}^{y}|p(y|\bm{x})\\
\nonumber
&=\frac{2}{2q(y_{\mathrm{max}}-y_{\mathrm{min}})+q^2}\frac{q^2}{2}p(y|\bm{x}) \quad (\because |\mathcal{S}_{q}^{y}| = \frac{q^2}{2})\\
\nonumber
&=\frac{q}{2(y_{\mathrm{max}}-y_{\mathrm{min}})+q}p(y|\bm{x})
\nonumber
\end{align}
By integrating $y$ on both sides, we can obtain
\begin{gather}
\nonumber
\mathrm{Pr}[y \in S|\bm{x}] = \frac{q}{2(y_{\mathrm{max}}-y_{\mathrm{min}})+q}
\end{gather}
which concludes the proof of Lemma \ref{lem1}.\qed

\subsection{Prove of Theorem \ref{thm2}}
\label{A.3}
Let us express $p(S|y\in S,\bm{x})$ as
\begin{align}
\nonumber
p(S|y\in S,\bm{x})&=\frac{p(y\in S,S|\bm{x})}{\mathrm{Pr}[y\in S|\bm{x}]}\\
\nonumber
&=\frac{\mathrm{Pr}[y\in S|S,\bm{x}]p(S|\bm{x})}{\mathrm{Pr}[y\in S|\bm{x}]}\\
\nonumber
&=\frac{\mathrm{Pr}[y\in S|S,\bm{x}]p(S)}{\mathrm{Pr}[y\in S|\bm{x}]}
\end{align}
where the last equality holds due to the fact that for each instance $\bm{x}$, $S$ is uniformly and randomly chosen. Since $p(S)=\frac{1}{|S_{q}|}$ if $q$ is specific. We can easily know that when $\underline{y} \in [y_{\mathrm{min}}-q,y_{\mathrm{min}}]$, $\overline{y} \in [y_{\mathrm{min}},\underline{y}+q]$, when $\underline{y} \in [y_{\mathrm{min}},y_{\mathrm{max}}]$, $\overline{y} \in [\underline{y},\underline{y}+q]$, so $|S_{q}|= \int_{y_{\mathrm{min}}-q}^{y_{\mathrm{min}}}\int_{y_{\mathrm{min}}}^{\underline{y}+q}1\ \mathrm{d} \overline{y}\mathrm{d} \underline{y}+ \int_{y_{\mathrm{min}}}^{y_{\mathrm{max}}}\int_{\underline{y}}^{\underline{y}+q}1\ \mathrm{d} \overline{y}\mathrm{d} \underline{y} = \frac{1}{2}q^2 + q(y_{\mathrm{max}}-y_{\mathrm{min}}) $, we have

\begin{align}
\nonumber
p(S|y\in S,\bm{x})&=\frac{\mathrm{Pr}[y\in S|S,\bm{x}]p(S)}{\mathrm{Pr}[y\in S|\bm{x}]}\\
\nonumber
&=\frac{2}{2q(y_{\mathrm{max}}-y_{\mathrm{min}})+q^2}\frac{\mathrm{Pr}[y\in S|S,\bm{x}]}{\mathrm{Pr}[y\in S|\bm{x}]}\\
\nonumber
&=\frac{2}{2q(y_{\mathrm{max}}-y_{\mathrm{min}})+q^2}\frac{2(y_{\mathrm{max}}-y_{\mathrm{min}}+q)}{q}\mathrm{Pr}[y\in S|S,\bm{x}] \quad (by\ Lemma\ \ref{lem1})\\
\nonumber
&=\frac{2}{q^2}\mathrm{Pr}[y\in S|S,\bm{x}]\\
\nonumber
&=\int_{y\in S}\frac{2}{q^2}p(y|\bm{x}) \mathrm{d}y\\
\nonumber
\end{align}
By multiplying $p(\bm{x})$ on both side, we have
\begin{align}
\nonumber
p(\bm{x},S|y\in S)&=\int_{y\in S}\frac{2}{q^2}p(\bm{x},y)\mathrm{d}y\\
\nonumber
&=\int_{y_l}^{y_r}\frac{2}{q^2}p(\bm{x},y)\mathrm{d}y\\
\nonumber
&=\tilde{p}(\bm{x},S)
\end{align}
which concludes the proof of Theorem \ref{thm2}.\qed
\section{Proofs of The Model Consistent}
\subsection{Prove of Theorem \ref{op-lm}}
\label{B.1}
First, we prove that the optimal model $f^\star$ learned from ordinary regression expected risk (\ref{expected_regression}) is also the optimal model for $R_{LM}(f)$ as follows. 
\begin{equation}
\begin{split}
R_{\mathrm{LM}}(f^\star)
	   &=\mathbb{E}_{\tilde{p}(\bm{x},S)}[\ell_{\mathrm{LM}}(f^{\star}(\bm{x}),S)]\\
	   &=\int_{\mathcal{X}}\int_{\mathcal{S}} \tilde{p}(\bm{x},S)\ell_{\mathrm{LM}}(f^{\star}(\bm{x}),S) \mathrm{d}S \mathrm{d}\bm{x}\\
	   &=\int_{\mathcal{X}}\int_{\mathcal{S}}\int_{\mathcal{Y}} p(\bm{x},y,S)\ell_{\mathrm{LM}}(f^{\star}(\bm{x}),S)\mathrm{d}y \mathrm{d}S \mathrm{d}\bm{x}\\
	   &=\int_{\mathcal{X}}\int_{\mathcal{S}}\int_{\mathcal{Y}} p(S|\bm{x},y)p(y|\bm{x})p(\bm{x})\ell_{\mathrm{LM}}(f^{\star}(\bm{x}),S)\mathrm{d}y \mathrm{d}S \mathrm{d}\bm{x}\\
	   &=\int_{\mathcal{X}}\int_{\mathcal{Y}}p(y|\bm{x})p(\bm{x})\ell(f^{\star}(\bm{x}),y)\mathrm{d}y \mathrm{d}\bm{x}\int_{\mathcal{S}} p(S|\bm{x},y) \mathrm{d}S\\
	   &=\int_{\mathcal{X}}\int_{\mathcal{Y}}\ell(f^{\star}(\bm{x}),y)p(\bm{x},y)\mathrm{d}y \mathrm{d}\bm{x}\\
	   &=R(f^{\star})=0
\end{split}
\end{equation}
where we used the equality $\ell_{\mathrm{LM}}(f^{\star}(\bm{x}),S) = \ell(f^{\star}(\bm{x}),y)$. This because when the true label $y \in S$, $\ell_{\mathrm{LM}}(f^{\star}(\bm{x}),S) = \mathbb{I}_{\{\underline{y}-f^{\star}(\bm{x})>0\}} + \mathbb{I}_{\{f^{\star}(\bm{x})-\overline{y}>0\}} = \ell(f^{\star}(\bm{x}),y)=0$. Therefore $f^{\star}$ is the optimal model for $R_{\mathrm{LM}}$.

On the other hand, we prove that $f^{\star}$ is the sole optimal model for $R_{\mathrm{LM}}$ by contradiction. Specifically, we assume that there is at least one other model $g$ that makes $R_{\mathrm{LM}}(g)=0$ and predicts a label $y_g \neq y$ for at least one instance $\bm{x}$. Therefore, for any $S$ containing $y_g$ we have
\begin{gather}
\ell_{\mathrm{LM}}(g(\bm{x}),S)=\mathbb{I}_{\{\underline{y}-y_{g}>0\}} + \mathbb{I}_{\{y_{g}-\overline{y}>0\}} = 0
\end{gather}

Nevertheless, the above equality could be always true on the condition that $y_{g}$ is invariably included in the interval $S$ of $\bm{x}$. In the problem setting, there is no other false label that \emph{always} occurs with true label in the interval $S$. Therefore, there is one, and only one minimizer of $R_{\mathrm{LM}}$, which is the same as the minimizer $f^\star$ learned from fully labeled data. The proof is completed.\qed
\subsection{Prove of Theorem \ref{lm-su}}
\label{B.2}
First, we prove that the optimal model $f_{\mathrm{LM}}^\star$ learned from limiting method expected risk (\ref{risk_lm}) is also the optimal model for $R^{\psi}_{\mathrm{LM}}(f)$ as follows.
\begin{equation}
\begin{split}
R^{\psi}_{\mathrm{LM}}(f^{\star}_{\mathrm{LM}})
	   &=\mathbb{E}_{\tilde{p}(\bm{x},S)}[\psi_{\mathrm{LM}}(f^{\star}_{\mathrm{LM}}(\bm{x}),S)]\\
	   &=\int_{\mathcal{X}}\int_{\mathcal{S}} \tilde{p}(\bm{x},S)\psi_{\mathrm{LM}}(f^{\star}_{\mathrm{LM}}(\bm{x}),S) \mathrm{d}S \mathrm{d}\bm{x}\\
	   &=\int_{\mathcal{X}}\int_{\mathcal{S}}\int_{\mathcal{Y}} p(\bm{x},y,S)\psi_{\mathrm{LM}}(f^{\star}_{\mathrm{LM}}(\bm{x}),S)\mathrm{d}y \mathrm{d}S \mathrm{d}\bm{x}\\
	   &=\int_{\mathcal{X}}\int_{\mathcal{S}}\int_{\mathcal{Y}} p(S|\bm{x},y)p(y|\bm{x})p(\bm{x})\psi_{\mathrm{LM}}(f^{\star}_{\mathrm{LM}}(\bm{x}),S)\mathrm{d}y \mathrm{d}S \mathrm{d}\bm{x}\\
	   &=\int_{\mathcal{X}}\int_{\mathcal{Y}}p(y|\bm{x})p(\bm{x})\ell_{\mathrm{LM}}(f^{\star}_{\mathrm{LM}}(\bm{x}),y)\mathrm{d}y \mathrm{d}\bm{x}\int_{\mathcal{S}} p(S|\bm{x},y) \mathrm{d}S\\
	   &=\int_{\mathcal{X}}\int_{\mathcal{Y}}\ell_{\mathrm{LM}}(f^{\star}_{\mathrm{LM}}(\bm{x}),y)p(\bm{x},y)\mathrm{d}y \mathrm{d}\bm{x}\\
	   &=R_{\mathrm{LM}}(f^{\star}_{\mathrm{LM}})=0
\end{split}
\end{equation}
where we used the equality $\psi_{\mathrm{LM}}(f^{\star}_{\mathrm{LM}}(\bm{x}),S) = \ell_{\mathrm{LM}}(f^{\star}_{\mathrm{LM}}(\bm{x}),y)$. This because when the true label $y \in S$, $\psi_{\mathrm{LM}}(f^{\star}_{\mathrm{LM}}(\bm{x}),S) =\max(0,\underline{y}-f^{\star}_{\mathrm{LM}}(\bm{x}))+\max(0,f^{\star}_{\mathrm{LM}}(\bm{x})-\overline{y}) = \ell_{\mathrm{LM}}(f^{\star}_{\mathrm{LM}}(\bm{x}),y)=0$. Therefore $f^{\star}_{\mathrm{LM}}$ is the optimal model for $R^{\psi}_{\mathrm{LM}}$.

On the other hand, we prove that $f^{\star}_{\mathrm{LM}}$ is the sole optimal model for $R^{\psi}_{\mathrm{LM}}$ by contradiction. Specifically, we assume that there is at least one other model $h$ that makes $R^{\psi}_{\mathrm{LM}}(h)=0$ and predicts a label $y_h \neq y$ for at least one instance $\bm{x}$. Therefore, for any $S$ containing $y_h$ we have
\begin{gather}
\psi_{\mathrm{LM}}(h(\bm{x}),S)=\mathrm{max}(0,\underline{y}-y_h)+\mathrm{max}(0,y_h-\overline{y})= 0
\end{gather}

Nevertheless, the above equality could be always true on the condition that $y_{h}$ is invariably included in the interval $S$ of $\bm{x}$. In the problem setting, there is no other false label that \emph{always} occurs with true label in the interval $S$. Therefore, there is one, and only one minimizer of $R^{\psi}_{\mathrm{LM}}$, which is the same as the minimizer $f^{\star}_{\mathrm{LM}}$ learned from limiting method. By Theorem \ref{op-lm}, $f^{\star}_{\mathrm{LM}}$ is the same as the minimizer $f^\star$ learned from fully labeled data. The proof is completed.\qed
\section{Proof of Corollary \ref{mae_surr}}
For any interval instance $(\bm{x},S)$, We consider three possible cases of model prediction: the predicted value is on the left side of the interval ($f(x)<\underline{y}$), the predicted value is on the right side of the interval ($f(x)>\overline{y}$) and the predicted value is exactly inside the interval ($\underline{y}\leq f(x)\leq \overline{y}$).

If the predicted value of the model lie on the left side of the interval, the losses of AVGL\_MAE and the surrogate method are as follows.
\begin{align}
\nonumber
\ell_{\mathrm{avgl\_mae}}(f(\bm{x}), S) &= \frac{1}{2}(|f(\bm{x})-\underline{y}| +|f(\bm{x})- \overline{y})|)=\frac{1}{2}(\underline{y}-f(\bm{x}) +\overline{y}-f(\bm{x}))=\frac{\underline{y} +\overline{y}}{2}-f(\bm{x}),\\
\nonumber
\psi_{\mathrm{LM}}(f(\bm{x}), S) &=\max(0,\underline{y}-f(\bm{x}))+\max(0,f(\bm{x})-\overline{y})=\underline{y}-f(\bm{x})+0=\underline{y}-f(\bm{x}).
\end{align}

If the predicted value of the model lie on the right side of the interval, the losses of AVGL\_MAE and the surrogate method are as follows.
\begin{align}
\nonumber
\ell_{\mathrm{avgl\_mae}}(f(\bm{x}), S) &= \frac{1}{2}(|f(\bm{x})-\underline{y}| +|f(\bm{x})- \overline{y})|)=\frac{1}{2}(f(\bm{x})-\underline{y} +f(\bm{x})-\overline{y})=f(\bm{x})-\frac{\underline{y} +\overline{y}}{2},\\
\nonumber
\psi_{\mathrm{LM}}(f(\bm{x}), S) &=\max(0,\underline{y}-f(\bm{x}))+\max(0,f(\bm{x})-\overline{y})=0+f(\bm{x})-\overline{y}=f(\bm{x})-\overline{y}.
\end{align}

If the predicted value of the model lie in the interval, the losses of AVGL\_MAE and the surrogate method are as follows.
\begin{align}
\nonumber
\ell_{\mathrm{avgl\_mae}}(f(\bm{x}), S) &= \frac{1}{2}(|f(\bm{x})-\underline{y}| +|f(\bm{x})- \overline{y})|)=0\\
\nonumber
\psi_{\mathrm{LM}}(f(\bm{x}), S) &=\max(0,\underline{y}-f(\bm{x}))+\max(0,f(\bm{x})-\overline{y})=0
\end{align}
We can see that the losses of AVGL\_MAE and the surrogate loss differ only in constant terms on the three possible cases. If our training model uses gradient descent, the gradients of AVGL\_MAE and the surrogate method are the same on all three possible cases.\qed
\section{Proof of Theorem \ref{error_bound}}
\label{errorbound}
Before directly proving Theorem \ref{error_bound}, we first introduce the following lemma.
\begin{lemma}
\label{est_lemma}
Let $\widehat{f}$ be the empirical risk minimizer (i.e., $\widehat{f}=\argmin_{f\in\mathcal{F}}\widehat{R}(f)$) and $f^\star$ be the true risk minimizer (i.e., $f^\star=\argmin_{f\in\mathcal{F}}R(f)$), then the following inequality holds:
\begin{gather}
\nonumber
R(\widehat{f})-R(f^\star)\leq 2\sup_{f\in\mathcal{F}}\left|\widehat{R}(f)-R(f)\right|.
\end{gather}
\end{lemma}
\begin{proof}
It is intuitive to obtain
\begin{align}
\nonumber
R(\widehat{f})-R(f^\star)&\leq R(\widehat{f})-\widehat{R}(\widehat{f})+\widehat{R}(\widehat{f}) - R(f^\star)\\
\nonumber
&\leq R(\widehat{f})-\widehat{R}(\widehat{f})+R(\widehat{f}) - R(f^\star)\\
\nonumber
&\leq 2\sup_{f\in\mathcal{F}}\left|\widehat{R}(f)-R(f)\right|,
\end{align}
which completes the proof. The same proof has been provided in \citet{mohri2012foundations}.
\end{proof}
Recall that $R^{\psi}_{\mathrm{LM}}(f)$ is denoted by 
\begin{align}
\nonumber
R^{\psi}_{\mathrm{LM}}(f)&=\mathbb{E}_{p(\bm{x},\underline{y},\overline{y})}\big[\psi(\underline{y}-f(\bm{x}))+\psi(f(\bm{x})-\overline{y})\big]\\
\nonumber
&=\mathbb{E}_{p(\bm{x},\underline{y},\overline{y})}\big[\max(\underline{y}-f(\bm{x}), 0) + \max(f(\bm{x})-\overline{y}, 0)\big]\\
\nonumber
&=\mathbb{E}_{p(\bm{x},\underline{y},\overline{y})}\big[\max(\underline{y}-f(\bm{x}), 0)\big] + \mathbb{E}_{p(\bm{x},\underline{y},\overline{y})}\big[\max(f(\bm{x})-\overline{y}, 0)\big]\\
\nonumber
&=R^{\psi, l}_{\mathrm{LM}}(f) + R^{\psi, r}_{\mathrm{LM}}(f),
\end{align}
where we have introduced $R^{\psi, l}_{\mathrm{LM}}(f) = \mathbb{E}_{p(\bm{x},\underline{y},\overline{y})}\big[\max(\underline{y}-f(\bm{x}), 0)\big]$ and $R^{\psi, r}_{\mathrm{LM}}(f)=\mathbb{E}_{p(\bm{x},\underline{y},\overline{y})}\big[\max(f(\bm{x})-\overline{y}, 0)\big]$ in the last equality.
In this way, we have 
\begin{align}
\nonumber
R^{\psi}_{\mathrm{LM}}(\widehat{f}_{\mathrm{LM}}) - R^{\psi}_{\mathrm{LM}}(f^\star) &= R^{\psi}_{\mathrm{LM}}(\widehat{f}_{\mathrm{LM}}) - R^{\psi}_{\mathrm{LM}}(f^\star_{\mathrm{LM}}) \\
\nonumber
&\leq 2\sup_{f\in\mathcal{F}}\left|R^{\psi}_{\mathrm{LM}}(\widehat{f}_{\mathrm{LM}}) - R^{\psi}_{\mathrm{LM}}(f^\star_{\mathrm{LM}})\right|\\
\nonumber
&\leq 2\sup_{f\in\mathcal{F}}\left|R^{\psi,l}_{\mathrm{LM}}(\widehat{f}_{\mathrm{LM}}) - R^{\psi,l}_{\mathrm{LM}}(f^\star_{\mathrm{LM}})\right| + 2\sup_{f\in\mathcal{F}}\left|R^{\psi,r}_{\mathrm{LM}}(\widehat{f}_{\mathrm{LM}}) - R^{\psi,r}_{\mathrm{LM}}(f^\star_{\mathrm{LM}})\right|
\end{align}
where the first equality holds, and the last inequality means that we can directly bound $\sup_{f\in\mathcal{F}}\left|R^{\psi,l}_{\mathrm{LM}}(\widehat{f}_{\mathrm{LM}}) - R^{\psi,l}_{\mathrm{LM}}(f^\star_{\mathrm{LM}})\right|$ and $\sup_{f\in\mathcal{F}}\left|R^{\psi,r}_{\mathrm{LM}}(\widehat{f}_{\mathrm{LM}}) - R^{\psi,r}_{\mathrm{LM}}(f^\star_{\mathrm{LM}})\right|$. Based on the assumptions introduced in Theorem \ref{error_bound} and using the discussion of Theorem 10.6 in \citet{mohri2012foundations}, with probability $1-\delta$,
\begin{align}
\nonumber
\sup_{f\in\mathcal{F}}\left|R^{\psi,l}_{\mathrm{LM}}(\widehat{f}_{\mathrm{LM}}) - R^{\psi,l}_{\mathrm{LM}}(f^\star_{\mathrm{LM}})\right|  &\leq M\sqrt{\frac{2d\log\frac{en}{d}}{n}} + M\sqrt{\frac{\log\frac{2}{\delta}}{2n}},\\
\nonumber
\sup_{f\in\mathcal{F}}\left|R^{\psi,r}_{\mathrm{LM}}(\widehat{f}_{\mathrm{LM}}) - R^{\psi,r}_{\mathrm{LM}}(f^\star_{\mathrm{LM}})\right| &\leq M^\prime\sqrt{\frac{2d^\prime\log\frac{en}{d^\prime}}{n}} + M^\prime\sqrt{\frac{\log\frac{2}{\delta}}{2n}}.
\end{align}
Therefore, with probability $1-\delta$,
\begin{align}
\nonumber
R^{\psi}_{\mathrm{LM}}(\widehat{f}_{\mathrm{LM}}) - R^{\psi}_{\mathrm{LM}}(f^\star) \leq 2M\sqrt{\frac{2d\log\frac{en}{d}}{n}} + 2M^\prime\sqrt{\frac{2d^\prime\log\frac{en}{d^\prime}}{n}} + 2(M+M^\prime)\sqrt{\frac{\log\frac{4}{\delta}}{2n}},
\end{align}
which completes the proof of Theorem \ref{error_bound}.
\section{Additional Information of Experiments}
\label{information}
\subsection{Details of Datasets}
\label{data}
In our experiments, we used AgeDB, IMDB-WIKI, STS-B and 6 UCI benchmark datasets including Abalone, Airfoil, Auto-mpg, Housing, Concrete and Power-plant. For each dataset, we follow the data distribution proposed in Section 3.1 to generate interval data. Then we randomly split each dataset into training, validation, and test sets by the proportions of 60\%, 20\%, and 20\%, respectively. Here, we provide the detailed information of these datasets we used in our experiments.

AgeDB is a regression dataset on age prediction collected by \cite{8014984}. It contains 16.4K face images with a minimum age of 0 and a maximum age of 101. We generated the interval regression dataset AgeDB-Interval at $q$ = 10, 20, 30, 40, and 50, respectively, and manually corrected the unreasonable intervals, such as intervals containing negative ages and intervals containing too old ages (less than 0 and greater than 150).

IMDB-WIKI is a regression dataset about age prediction collected by \cite{rothe2018deep}. It contains 523.0K face images, and we filtered the images that do not match the criteria and finally kept 213.5K images, where the minimum age is 0 years and the maximum age is 186 years. We generated the interval regression dataset IMDB-WIKI-Interval at $q$ = 20, 30, and 40, respectively, and manually corrected the unreasonable intervals, such as those containing negative ages and those containing too old ages (less than 0 and greater than 200).

Semantic Textual Similarity Benchmark (STS-B) \cite{cer2017semeval} is a collection of sentence pairs extracted from news headlines, video and image captions, and natural language inference data. Each sentence pair is scored for similarity by multiple annotators, and the final score is averaged as the final score. We created a dataset with 15.7K from \cite{yang2021delving}. We generated interval regression datasets for STS-B-Interval at $q$ = 3.0, 3.5, 4.0, 4.5 and 5.0, respectively.

We conducted experiments on 6 UCI benchmark datasets including Abalone, Airfoil, Auto-mpg, Housing, Concrete and Power-plant. All of these datasets can be downloaded from the UCI Machine Learning. Based on the span of the dataset labels, we selected two larger q values to generate interval regression data for each dataset.


\subsection{Evaluation Metrics}
We describe in detail all the evaluation metrics we used in our experiments.

\noindent\textbf{MSE.} The mean squared error (MSE) is defined as $\frac{1}{n}\sum_{i=1}^{n}(y_i - \widehat{y}_i)^2$, where $n$ denotes the number of samples, $y_i$ denotes the ground truth value, and $\widehat{y}_i$ denotes the predicted value. MSE represents the averaged squared difference between the ground truth and predicted values over all samples.

\noindent\textbf{MAE.} The mean absolute error (MAE) is defined as $\frac{1}{n}\sum_{i=1}^{n}|y_i - \widehat{y}_i|$, where $n$ denotes the number of samples, $y_i$ denotes the ground truth value, and $\widehat{y}_i$ denotes the predicted value. MAE represents the averaged absolute difference
between the ground truth and predicted values over all samples.

\noindent\textbf{GM.} We use the Geometric Mean (GM) proposed by \cite{yang2021delving} as our evaluation method, and is defined as $(\prod\nolimits_{i=1}^{n}e_{i})^{\frac{1}{n}}$, where $e_{i}\triangleq|y_i - \widehat{y}_i|$. GM is using the geometric mean to describe the fairness of the model predictions rather than the arithmetic mean.

\noindent\textbf{Pearson correlation.} Pearson correlation is an evaluation of the linear relationship between the predicted value and the ground truth value, and is defined as $\frac{\sum_{i=1}^{n}(y_i -\overline{y})(\widehat{y}_i -\overline{\widehat{y}})}{\sqrt{\sum_{i=1}^{n}(y_i - \overline{y})^2}\sqrt{\sum_{i=1}^{n}(\widehat{y}_i -\overline{\widehat{y}})^2}}$, where $\overline{y}$ denotes the average of all ground truth values, $\overline{\widehat{y}}$ denotes the average of all predicted values, i.e., $\overline{y} = \frac{1}{n}\sum_{i=1}^{n}y_i$, $\overline{\widehat{y}} = \frac{1}{n}\sum_{i=1}^{n}\widehat{y}_i$.
\begin{table}[!t]
\caption{Complete evaluation results on AgeDB}
\label{cagedb}
\resizebox{1.00\textwidth}{!}{
\setlength{\tabcolsep}{3.0mm}{
\begin{tabular}{cc|ccccc|ccccc|ccccc}
\toprule
\multicolumn{2}{c|}{Metric}                     & \multicolumn{5}{c|}{MSE}                                                                & \multicolumn{5}{c|}{MAE}                                                                & \multicolumn{5}{c}{GM}                                                                  \\ \hline
\multicolumn{2}{c|}{Approach}                    & $q=30$            & $q=40$            & $q=50$            & $q=60$            & $q=70$            & $q=30$            & $q=40$            & $q=50$            & $q=60$            & $q=70$            & $q=30$            & $q=40$            & $q=50$            & $q=60$            & $q=70$            \\ \hline
\multicolumn{2}{c|}{\multirow{2}{*}{Supervised}} & \multicolumn{5}{c|}{102.71}                                                             & \multicolumn{5}{c|}{7.82}                                                               & \multicolumn{5}{c}{5.22}                                                                \\
\multicolumn{2}{c|}{}                            & \multicolumn{5}{c|}{(3.12)}                                                             & \multicolumn{5}{c|}{(0.14)}                                                             & \multicolumn{5}{c}{(0.13)}                                                              \\ \hline
\multirow{6}{*}{LEFT} & \multirow{2}{*}{MAE} & 158.38 & 210.34 & 205.75 & 283.83 & 355.59 & 9.88 & 11.59 & 11.27 & 13.62 & 14.93 & 6.32 & 7.69 & 7.37 & 9.14 & 10.00 \\
 &  & (6.71) & (19.85) & (22.04) & (15.36) & (105.16) & (0.22) & (0.61) & (0.80) & (0.62) & (2.38) & (0.13) & (0.32) & (0.67) & (0.40) & (1.77) \\
 & \multirow{2}{*}{MSE} & 134.99 & 196.83 & 221.54 & 295.49 & 347.27 & 9.25 & 11.19 & 11.91 & 13.90 & 14.98 & 6.08 & 7.46 & 7.88 & 9.54 & 10.47 \\
 &  & (4.03) & (19.28) & (18.04) & (32.09) & (80.77) & (0.06) & (0.71) & (0.48) & (1.24) & (2.10) & (0.04) & (0.55) & (0.39) & (1.18) & (1.83) \\
 & \multirow{2}{*}{Huber} & 156.34 & 175.95 & 208.03 & 317.17 & 360.11 & 9.85 & 10.54 & 11.42 & 14.39 & 15.41 & 6.37 & 7.00 & 7.56 & 9.62 & 10.64 \\
 &  & (17.95) & (7.61) & (18.36) & (22.57) & (53.12) & (0.57) & (0.19) & (0.61) & (0.37) & (1.30) & (0.33) & (0.14) & (0.59) & (0.05) & (1.16) \\ \hline
\multirow{6}{*}{RIGHT} & \multirow{2}{*}{MAE} & 154.87 & 196.95 & 233.18 & 255.24 & 428.05 & 9.91 & 11.31 & 12.48 & 13.12 & 17.51 & 6.64 & 7.74 & 8.67 & 8.97 & 12.77 \\
 &  & (13.12) & (20.96) & (26.58) & (15.12) & (41.79) & (0.45) & (0.58) & (0.79) & (0.40) & (0.93) & (0.39) & (0.40) & (0.65) & (0.29) & (0.94) \\
 & \multirow{2}{*}{MSE} & 146.85 & 215.06 & 260.37 & 304.71 & 452.61 & 9.57 & 11.96 & 13.27 & 14.47 & 18.15 & 6.34 & 8.04 & 9.12 & 10.28 & 13.31 \\
 &  & (24.39) & (23.92) & (15.32) & (49.88) & (45.87) & (0.86) & (0.66) & (0.49) & (1.29) & (1.15) & (0.87) & (0.54) & (0.34) & (1.10) & (1.26) \\
 & \multirow{2}{*}{Huber} & 149.14 & 179.72 & 246.29 & 279.70 & 436.29 & 9.72 & 10.84 & 12.88 & 13.90 & 17.50 & 6.54 & 7.41 & 8.85 & 9.77 & 12.77 \\
 &  & (7.74) & (10.57) & (16.02) & (17.45) & (78.90) & (0.31) & (0.37) & (0.48) & (0.44) & (1.86) & (0.19) & (0.25) & (0.63) & (0.24) & (1.79) \\ \hline
\multirow{6}{*}{Middle} & \multirow{2}{*}{MAE} & 116.14 & 133.44 & 129.55 & 138.95 & 150.88 & 8.38 & 8.93 & 8.97 & 9.21 & 9.57 & 5.44 & 5.75 & 6.22 & 5.93 & 6.35 \\
 &  & (2.57) & (5.05) & (1.37) & (5.22) & (3.66) & (0.13) & (0.15) & (0.13) & (0.11) & (0.12) & (0.04) & (0.04) & (0.72) & (0.08) & (0.11) \\
 & \multirow{2}{*}{MSE} & 119.90 & 133.27 & 128.84 & 138.36 & 149.82 & 8.45 & 8.94 & 8.89 & 9.32 & 9.53 & 5.57 & 5.91 & 5.75 & 6.16 & 6.33 \\
 &  & (6.23) & (5.18) & (3.01) & (6.10) & (5.28) & (0.18) & (0.22) & (0.21) & (0.27) & (0.07) & (0.18) & (0.05) & (0.10) & (0.24) & (0.03) \\
 & \multirow{2}{*}{Huber} & 121.78 & 131.43 & 131.38 & 140.40 & 149.25 & 8.62 & 8.92 & 8.96 & 9.28 & 9.65 & 5.54 & 5.76 & 5.83 & 6.15 & 6.35 \\
 &  & (4.75) & (3.84) & (2.20) & (6.18) & (0.70) & (0.14) & (0.15) & (0.08) & (0.19) & (0.09) & (0.07) & (0.10) & (0.05) & (0.18) & (0.12) \\ \hline
\multicolumn{2}{c|}{\multirow{2}{*}{CRM}} & 221.66 & 303.52 & 398.50 & 523.53 & 653.81 & 12.18 & 14.57 & 17.17 & 20.11 & 22.89 & 8.42 & 10.51 & 12.82 & 15.86 & 18.63 \\
\multicolumn{2}{c|}{} & (1.45) & (11.12) & (14.80) & (4.63) & (17.30) & (0.07) & (0.30) & (0.40) & (0.08) & (0.09) & (0.07) & (0.24) & (0.44) & (0.20) & (0.21) \\
\multicolumn{2}{c|}{\multirow{2}{*}{RANN}} & 125.04 & 126.02 & 129.86 & 139.83 & 148.25 & 8.69 & 8.73 & 8.89 & 9.32 & 9.69 & 5.62 & 5.69 & 5.92 & 6.82 & 6.42 \\
\multicolumn{2}{c|}{} & (1.09) & (1.74) & (1.00) & (3.41) & (2.33) & (0.04) & (0.07) & (0.05) & (0.22) & (0.11) & (0.12) & (0.07) & (0.03) & (1.01) & (0.09) \\
\multicolumn{2}{c|}{\multirow{2}{*}{SINN}} & 218.16 & 302.06 & 404.80 & 524.74 & 649.27 & 12.11 & 14.54 & 17.41 & 19.97 & 22.70 & 8.39 & 10.45 & 13.32 & 15.69 & 18.52 \\
\multicolumn{2}{c|}{} & (1.53) & (6.53) & (2.70) & (7.17) & (2.94) & (0.08) & (0.24) & (0.04) & (0.13) & (0.14) & (0.14) & (0.29) & (0.13) & (0.24) & (0.25) \\
\multicolumn{2}{c|}{\multirow{2}{*}{IN}} & 118.85 & 133.86 & 138.58 & 147.95 & 152.34 & 8.46 & 9.00 & 9.15 & 9.53 & 9.65 & 5.75 & 5.88 & 5.94 & 6.23 & 6.30 \\
\multicolumn{2}{c|}{} & (5.20) & (6.46) & (4.06) & (2.82) & (9.30) & (0.17) & (0.13) & (0.09) & (0.07) & (0.21) & (0.25) & (0.09) & (0.05) & (0.16) & (0.15) \\ \hline
\multicolumn{2}{c|}{\multirow{2}{*}{LM}} & \textbf{115.76} & \textbf{121.24} & \textbf{123.51} & \textbf{128.04} & \textbf{129.47} & \textbf{8.36} & \textbf{8.67} & \textbf{8.75} & \textbf{8.82} & \textbf{8.92} & \textbf{5.41} & \textbf{5.61} & \textbf{5.65} & \textbf{5.73} & \textbf{5.84} \\
\multicolumn{2}{c|}{} & \textbf{(1.64)} & \textbf{(4.73)} & \textbf{(2.78)} & \textbf{(3.52)} & \textbf{(7.15)} & \textbf{(0.07)} & \textbf{(0.03)} & \textbf{(0.14)} & \textbf{(0.16)} & \textbf{(0.14)} & \textbf{(0.03)} & \textbf{(0.02)} & \textbf{(0.13)} & \textbf{(0.07)} & \textbf{(0.11)}
\\
\bottomrule
\end{tabular}
}
}
\vspace{-0.5cm}
\end{table}
\begin{table}[!t]
\caption{All results on IMDB-WIKI.}
\label{cimdb}
\resizebox{1.00\textwidth}{!}{
\setlength{\tabcolsep}{3.0mm}{
\begin{tabular}{cc|ccccc|ccccc|ccccc}
\toprule
\multicolumn{2}{c|}{Metrics}                     & \multicolumn{5}{c|}{MSE}                                                                & \multicolumn{5}{c|}{MAE}                                                                & \multicolumn{5}{c}{GM}                                                                  \\ \hline
\multicolumn{2}{c|}{Approach}                    & $q=40$            & $q=50$            & $q=60$            & $q=70$            & $q=80$            & $q=40$            & $q=50$            & $q=60$            & $q=70$            & $q=80$            & $q=40$            & $q=50$            & $q=60$            & $q=70$            & $q=80$            \\ \hline
\multicolumn{2}{c|}{\multirow{2}{*}{Supervised}} & \multicolumn{5}{c|}{123.82}                                                             & \multicolumn{5}{c|}{8.38}                                                               & \multicolumn{5}{c}{5.40}                                                                \\
\multicolumn{2}{c|}{}                            & \multicolumn{5}{c|}{(3.06)}                                                             & \multicolumn{5}{c|}{(0.07)}                                                             & \multicolumn{5}{c}{(0.09)}                                                              \\ \hline
\multirow{6}{*}{LEFT}   & \multirow{2}{*}{MAE}   & 270.15          & 336.73          & 402.05          & 494.84          & 581.54          & 13.06           & 14.76           & 16.28           & 18.30           & 20.39           & 8.60            & 10.09           & 13.54           & 12.98           & 15.28           \\
                        &                        & (15.14)         & (30.92)         & (39.13)         & (28.64)         & (19.81)         & (0.53)          & (0.77)          & (1.10)          & (0.82)          & (0.46)          & (0.63)          & (0.53)          & (2.46)          & (0.93)          & (0.69)          \\
                        & \multirow{2}{*}{MSE}   & 255.40          & 305.92          & 358.15          & 383.77          & 421.16          & 12.67           & 13.95           & 15.21           & 15.86           & 16.66           & 8.24            & 9.35            & 10.34           & 10.91           & 11.70           \\
                        &                        & (13.45)         & (7.33)          & (25.61)         & (16.82)         & (94.72)         & (0.40)          & (0.60)          & (0.57)          & (0.36)          & (2.36)          & (0.43)          & (0.29)          & (0.42)          & (0.45)          & (2.36)          \\
                        & \multirow{2}{*}{Huber} & 291.25          & 320.37          & 385.99          & 512.73          & 596.39          & 13.39           & 14.24           & 15.82           & 18.57           & 20.71           & 8.71            & 11.40           & 10.83           & 13.46           & 15.57           \\
                        &                        & (14.98)         & (10.01)         & (41.84)         & (107.26)        & (46.08)         & (0.26)          & (0.42)          & (1.14)          & (2.54)          & (0.79)          & (0.17)          & (3.00)          & (1.07)          & (2.67)          & (0.47)          \\ \hline
\multirow{6}{*}{RIGHT}  & \multirow{2}{*}{MAE}   & 197.26          & 265.60          & 328.77          & 608.51          & 679.26          & 11.15           & 13.38           & 14.74           & 20.80           & 22.82           & 8.07            & 10.74           & 13.29           & 15.69           & 17.55           \\
                        &                        & (17.24)         & (27.79)         & (63.99)         & (76.85)         & (80.82)         & (0.62)          & (0.75)          & (1.55)          & (1.50)          & (2.21)          & (0.55)          & (1.61)          & (3.84)          & (1.31)          & (3.04)          \\
                        & \multirow{2}{*}{MSE}   & 214.42          & 280.99          & 357.06          & 497.47          & 643.99          & 11.71           & 13.50           & 15.29           & 18.82           & 21.38           & 7.82            & 9.17            & 10.48           & 13.63           & 15.55           \\
                        &                        & (10.40)         & (12.52)         & (64.59)         & (43.98)         & (85.80)         & (0.35)          & (0.35)          & (1.54)          & (1.05)          & (1.49)          & (0.31)          & (0.32)          & (1.18)          & (0.96)          & (1.24)          \\
                        & \multirow{2}{*}{Huber} & 198.97          & 243.81          & 379.34          & 536.87          & 491.39          & 11.23           & 12.82           & 16.19           & 19.63           & 18.24           & 8.48            & 10.84           & 14.15           & 14.14           & 12.83           \\
                        &                        & (6.47)          & (11.64)         & (45.10)         & (21.50)         & (103.54)        & (0.21)          & (0.31)          & (1.01)          & (0.67)          & (2.22)          & (0.16)          & (0.16)          & (2.10)          & (0.48)          & (1.88)          \\ \hline
\multirow{6}{*}{Middle}    & \multirow{2}{*}{MAE}   & 140.29          & 148.93          & 152.20          & 154.79          & 157.17          & 8.96            & 9.24            & 9.47            & 9.55            & 9.76            & 5.94            & 6.20            & 6.82            & 7.43            & 6.96            \\
                        &                        & (6.68)          & (3.00)          & (9.59)          & (4.29)          & (5.96)          & (0.10)          & (0.08)          & (0.36)          & (0.21)          & (0.22)          & (0.38)          & (0.51)          & (0.06)          & (0.71)          & (0.61)          \\
                        & \multirow{2}{*}{MSE}   & 135.57          & 142.68          & 144.34          & 153.10          & 153.80          & 8.89            & 9.10            & 9.23            & 9.61            & 9.79            & 5.97            & 5.92            & 6.45            & 6.61            & 6.97            \\
                        &                        & (2.77)          & (2.43)          & (2.60)          & (9.34)          & (3.57)          & (0.10)          & (0.06)          & (0.11)          & (0.30)          & (0.08)          & (0.54)          & (0.20)          & (0.70)          & (0.16)          & (0.43)          \\
                        & \multirow{2}{*}{Huber} & 142.21          & 148.52          & 153.97          & 152.52          & 154.77          & 9.04            & 9.22            & 9.49            & 9.49            & 9.70            & 5.71            & 6.05            & 6.25            & 6.12            & 6.41            \\
                        &                        & (6.33)          & (3.93)          & (2.31)          & (3.48)          & (3.58)          & (0.19)          & (0.15)          & (0.05)          & (0.12)          & (0.15)          & (0.13)          & (0.17)          & (0.24)          & (0.10)          & (0.14)          \\ \hline
\multicolumn{2}{c|}{\multirow{2}{*}{CRM}}        & 302.50          & 380.12          & 519.7           & 602.65          & 740.99          & 14.38           & 16.81           & 20.11           & 21.85           & 24.61           & 10.18           & 12.81           & 13.26           & 18.16           & 21.34           \\
\multicolumn{2}{c|}{}                            & (12.53)         & (15.30)         & (10.23)         & (7.20)          & (19.56)         & (0.42)          & (0.28)          & (0.20)          & (0.16)          & (0.78)          & (0.56)          & (0.46)          & (0.44)          & (0.39)          & (0.56)          \\
\multicolumn{2}{c|}{\multirow{2}{*}{RANN}}       & 137.37          & 140.79          & 145.40          & 150.11          & 166.33          & 8.98            & 8.98            & 9.32            & 9.52            & 10.09           & 5.82            & 6.14            & 6.74            & 6.85            & 6.63            \\
\multicolumn{2}{c|}{}                            & (3.09)          & (2.28)          & (2.20)          & (2.48)          & (4.02)          & (0.10)          & (0.12)          & (0.08)          & (0.07)          & (0.08)          & (0.12)          & (0.14)          & (0.18)          & (0.25           & (0.10)          \\
\multicolumn{2}{c|}{\multirow{2}{*}{SINN}}       & 314.49          & 385.29          & 515.80          & 629.51          & 754.28          & 14.79           & 16.73           & 19.84           & 22.36           & 24.73           & 10.88           & 12.59           & 15.77           & 17.31           & 20.28           \\
\multicolumn{2}{c|}{}                            & (4.08)          & (8.26)          & (9.95)          & (11.10)         & (15.20)         & (0.14)          & (0.19)          & (0.30)          & (0.23)          & (0.65)          & (0.23)          & (0.45)          & (0.55)          & (0.72)          & (0.49)          \\
\multicolumn{2}{c|}{\multirow{2}{*}{IN}}         & 147.09          & 148.71          & 152.66          & 155.73          & 156.04          & 9.25            & 9.28            & 9.51            & 9.59            & 9.79            & 6.45            & 7.08            & 6.60            & 6.90            & 7.61            \\
\multicolumn{2}{c|}{}                            & (0.59)          & (1.01)          & (2.54)          & (5.67)          & (2.89)          & (0.04)          & (0.09)          & (0.10)          & (0.17)          & (0.06)          & (0.24)          & (0.36)          & (0.42)          & (0.70)          & (0.15)          \\ \hline
\multicolumn{2}{c|}{\multirow{2}{*}{LM}}         & \textbf{133.98} & \textbf{134.15} & \textbf{141.45} & \textbf{148.19} & \textbf{146.52} & \textbf{8.75}   & \textbf{8.83}   & \textbf{9.07}   & \textbf{9.40}   & \textbf{9.39}   & \textbf{5.59}   & \textbf{5.79}   & \textbf{5.81}   & \textbf{6.10}   & \textbf{6.12}   \\
\multicolumn{2}{c|}{}                            & \textbf{(1.57)} & \textbf{(2.49)} & \textbf{(2.32)} & \textbf{(2.44)} & \textbf{(5.00)} & \textbf{(0.06)} & \textbf{(0.07)} & \textbf{(0.11)} & \textbf{(0.08)} & \textbf{(0.04)} & \textbf{(0.08)} & \textbf{(0.20)} & \textbf{(0.07)} & \textbf{(0.10)} & \textbf{(0.14)}
\\
\bottomrule
\end{tabular}
}
}
\end{table}
\begin{table}[!t]
\caption{All results on STS-B.}
\label{csts}
\resizebox{1.00\textwidth}{!}{
\setlength{\tabcolsep}{3.0mm}{
\begin{tabular}{cc|ccccc|ccccc|ccccc}
\toprule
\multicolumn{2}{c|}{Metrics}                     & \multicolumn{5}{c|}{MSE}                                                                & \multicolumn{5}{c|}{MAE}                                                                & \multicolumn{5}{c}{Pearson}                                                                   \\ \hline
\multicolumn{2}{c|}{Approach}                    & $q=3.0$           & $q=3.5$           & $q=4.0$           & $q=4.5$           & $q=5.0$           & $q=3.0$           & $q=3.5$           & $q=4.0$           & $q=4.5$           & $q=5.0$           & $q=3.0$           & $q=3.5$           & $q=4.0$           & $q=4.5$           & $q=5.0$           \\ \hline
\multicolumn{2}{c|}{\multirow{2}{*}{Supervised}} & \multicolumn{5}{c|}{1.16}                                                               & \multicolumn{5}{c|}{0.87}                                                               & \multicolumn{5}{c}{0.71}                                                                \\
\multicolumn{2}{c|}{}                            & \multicolumn{5}{c|}{(0.05)}                                                             & \multicolumn{5}{c|}{(0.02)}                                                             & \multicolumn{5}{c}{(0.01)}                                                              \\ \hline
\multirow{6}{*}{LEFT} & \multirow{2}{*}{MAE} & 1.54 & 1.88 & 2.35 & 2.63 & 3.03 & 1.01 & 1.11 & 1.26 & 1.33 & 1.43 & \textbf{0.69} & \textbf{0.66} & 0.64 & 0.62 & 0.59 \\
 &  & (0.06) & (0.06) & (0.15) & (0.13) & (0.06) & (0.02) & (0.02) & (0.05) & (0.03) & (0.02) & \textbf{(0.01)} & \textbf{(0.01)} & (0.01) & (0.02) & (0.01) \\
 & \multirow{2}{*}{MSE} & 2.02 & 2.16 & 2.55 & 2.88 & 3.13 & 1.17 & 1.20 & 1.32 & 1.40 & 1.47 & 0.66 & 0.65 & 0.63 & 0.61 & 0.58 \\
 &  & (0.08) & (0.07) & (0.12) & (0.11) & (0.12) & (0.03) & (0.02) & (0.03) & (0.03) & (0.04) & (0.01) & (0.01) & (0.01) & (0.01) & (0.03) \\
 & \multirow{2}{*}{Huber} & 1.92 & 2.26 & 2.51 & 2.76 & 3.05 & 1.14 & 1.23 & 1.31 & 1.37 & 1.44 & 0.67 & 0.65 & 0.63 & 0.61 & 0.59 \\
 &  & (0.13) & (0.10) & (0.14) & (0.10) & (0.15) & (0.04) & (0.03) & (0.05) & (0.03) & (0.04) & (0.01) & (0.01) & (0.01) & (0.01) & (0.02) \\ \hline
\multirow{6}{*}{RIGHT} & \multirow{2}{*}{MAE} & 1.75 & 1.87 & 2.09 & 2.05 & 2.31 & 1.06 & 1.12 & 1.17 & 1.16 & 1.23 & 0.68 & 0.63 & 0.62 & 0.55 & 0.55 \\
 &  & (0.14) & (0.19) & (0.18) & (0.28) & (0.40) & (0.04) & (0.07) & (0.05) & (0.08) & (0.10) & (0.01) & (0.05) & (0.02) & (0.09) & (0.05) \\
 & \multirow{2}{*}{MSE} & 1.53 & 1.74 & 1.78 & 2.00 & 2.26 & 1.00 & 1.06 & 1.08 & 1.14 & 1.21 & 0.63 & 0.60 & 0.56 & 0.53 & 0.52 \\
 &  & (0.05) & (0.13) & (0.19) & (0.17) & (0.16) & (0.02) & (0.03) & (0.06) & (0.05) & (0.03) & (0.01) & (0.02) & (0.05) & (0.02) & (0.02) \\
 & \multirow{2}{*}{Huber} & 1.56 & 1.77 & 1.86 & 1.91 & 2.39 & 1.01 & 1.07 & 1.10 & 1.12 & 1.24 & 0.64 & 0.62 & 0.58 & 0.56 & 0.52 \\
 &  & (0.08) & (0.11) & (0.27) & (0.15) & (0.08) & (0.03) & (0.04) & (0.08) & (0.05) & (0.03) & (0.02) & (0.01) & (0.02) & (0.02) & (0.02) \\ \hline
\multirow{6}{*}{Middle} & \multirow{2}{*}{MAE} & \textbf{1.18} & 1.23 & \textbf{1.27} & \textbf{1.31} & 1.41 & \textbf{0.88} & 0.91 & \textbf{0.92} & \textbf{0.94} & 0.98 & \textbf{0.69} & \textbf{0.66} & \textbf{0.66} & \textbf{0.64} & \textbf{0.62} \\
 &  & \textbf{(0.03)} & (0.05) & \textbf{(0.07)} & \textbf{(0.03)} & (0.05) & \textbf{(0.01)} & (0.02) & \textbf{(0.03)} & \textbf{(0.01)} & (0.02) & \textbf{(0.01)} & \textbf{(0.02)} & \textbf{(0.01)} & \textbf{(0.01)} & \textbf{(0.02)} \\
 & \multirow{2}{*}{MSE} & 1.31 & 1.33 & 1.41 & 1.41 & 1.49 & 0.94 & 0.95 & 0.98 & 0.99 & 1.02 & 0.65 & 0.64 & 0.62 & 0.61 & 0.59 \\
 &  & (0.02) & (0.05) & (0.04) & (0.01) & (0.03) & (0.01) & (0.02) & (0.01) & (0.01) & (0.01) & (0.01) & (0.02) & (0.01) & (0.01) & (0.01) \\
 & \multirow{2}{*}{Huber} & 1.37 & 1.30 & 1.37 & 1.39 & 1.48 & 0.96 & 0.94 & 0.97 & 0.97 & 1.01 & 0.64 & 0.64 & 0.62 & 0.62 & 0.59 \\
 &  & (0.16) & (0.04) & (0.08) & (0.02) & (0.03) & (0.06) & (0.02) & (0.03) & (0.01) & (0.01) & (0.03) & (0.01) & (0.01) & (0.01) & (0.01) \\ \hline
\multicolumn{2}{c|}{\multirow{2}{*}{CRM}} & 1.36 & 1.50 & 1.53 & 1.52 & 1.56 & 0.96 & 1.00 & 1.02 & 1.02 & 1.05 & 0.64 & 0.60 & 0.59 & 0.58 & 0.56 \\
\multicolumn{2}{c|}{} & (0.02) & (0.10) & (0.08) & (0.06) & (0.06) & (0.01) & (0.04) & (0.03) & (0.02) & (0.03) & (0.01) & (0.03) & (0.03) & (0.03) & (0.02) \\
\multicolumn{2}{c|}{\multirow{2}{*}{RANN}} & 1.44 & 1.52 & 1.55 & 1.63 & 1.62 & 0.98 & 1.01 & 1.02 & 1.05 & 1.06 & 0.59 & 0.57 & 0.55 & 0.52 & 0.51 \\
\multicolumn{2}{c|}{} & (0.05) & (0.04) & (0.06) & (0.06) & (0.06) & (0.02) & (0.02) & (0.03) & (0.02) & (0.03) & (0.01) & (0.02) & (0.02) & (0.02) & (0.02) \\
\multicolumn{2}{c|}{\multirow{2}{*}{SINN}} & 1.37 & 1.50 & 1.50 & 1.54 & 1.61 & 0.96 & 1.00 & 1.01 & 1.03 & 1.06 & 0.64 & 0.60 & 0.60 & 0.58 & 0.54 \\
\multicolumn{2}{c|}{} & (0.04) & (0.05) & (0.10) & (0.05) & (0.09) & (0.01) & (0.01) & (0.03) & (0.02) & (0.03) & (0.01) & (0.01) & (0.02) & (0.01) & (0.03) \\
\multicolumn{2}{c|}{\multirow{2}{*}{IN}} & 1.26 & 1.39 & 1.37 & 1.37 & 1.41 & 0.91 & 0.96 & 0.96 & 0.96 & 0.98 & 0.67 & 0.62 & 0.62 & 0.62 & 0.61 \\
\multicolumn{2}{c|}{} & (0.12) & (0.21) & (0.14) & (0.08) & (0.07) & (0.04) & (0.06) & (0.04) & (0.03) & (0.03) & (0.04) & (0.06) & (0.05) & (0.01) & (0.01) \\ \hline
\multicolumn{2}{c|}{\multirow{2}{*}{LM}} & \textbf{1.19} & \textbf{1.22} & \textbf{1.27} & \textbf{1.31} & \textbf{1.37} & \textbf{0.88} & \textbf{0.90} & 0.93 & \textbf{0.94} & \textbf{0.96} & \textbf{0.69} & \textbf{0.66} & 0.65 & \textbf{0.64} & \textbf{0.62} \\
\multicolumn{2}{c|}{} & \textbf{(0.04)} & \textbf{(0.02)} & \textbf{(0.04)} & \textbf{(0.04)} & \textbf{(0.05)} & \textbf{(0.01)} & \textbf{(0.01)} & (0.01) & \textbf{(0.01)} & \textbf{(0.01)} & \textbf{(0.01)} & \textbf{(0.01)} & (0.01) & \textbf{(0.02)} & \textbf{(0.02)}
\\
\bottomrule
\end{tabular}
}
}
\end{table}
\section{Additional Results}
\label{all_result}
In the experiment of main paper, we show some of the experiments on the three datasets AgeDB, IMDB-WIKI and STS-B. Here, we provide the complete evaluation results on the nine used datasets, which include more evaluation metrics in addition to the results in the main paper.
\subsection{Complete Results on AgeDB}
We show the complete results of AgeDB in Table \ref{cagedb}. In the table, we show the test performance (mean and std) of each method with ResNet-50, evaluated using MSE, MAE and GM. We repeat the sampling-and-training process 3 times. The best performance is highlighted in bold. In addition, Our Vs Sup means the mean error between our proposed method and the supervised method (Our - Supervised). Specifically, we use red color to indicate that our method is weaker than the supervised method and green color to indicate that our method is better than the supervised method. As the table shows, our proposed method LM has significant advantages in all evaluation metrics.
\subsection{Complete Results on IMDB-WIKI}

We show the complete results of IMDB-WIKI in Table \ref{cimdb}. Similar to AgeDB, we evaluate each method using MSE, MAE and GM. We repeat the sampling-and-training process 3 times. It is worth noting that we chose different q values for AgeDB and IMDB-WIKI, although both are age predictions. IMDB-WIKI has a larger training set, and we want to test our method on a large dataset in a harsh environment (large $q$). As shown in the table \ref{cimdb}, our proposed method LM has significant advantages in all evaluation metrics. Even in the case of $q=80$, the performance does not degrade excessively after learning from a large number of training sets.
\subsection{Complete Results on STS-B}
We show the complete results of STS-B in Table \ref{csts}. In the table, we show the test performance (mean and std) of each method with BiLSTM + GloVe word embeddings baseline, evaluated using MSE, MAE and Pearson. Unlike AgeDB and IMDB-WIKI, STS-B has a smaller span of labels, so we can only choose smaller $q$ values to generate interval data. As shown in the table, the difference between methods is not significant when $q$ is small. As $q$ keeps increasing, all the methods tend to decrease in performance, while our method decreases more slowly and has a significant advantage at large $q$.
\begin{table*}[!t]
\caption{Comparison of MMIT and LM with MLP model}
\label{mmit}
\resizebox{1.00\textwidth}{!}{
\setlength{\tabcolsep}{3.5mm}{
\begin{tabular}{cc|ccc|ccc|ccc|ccc}
\toprule
\multicolumn{2}{c|}{Metrics} & \multicolumn{3}{c|}{MSE} & \multicolumn{3}{c|}{MAE} & \multicolumn{3}{c|}{Pearson} & \multicolumn{3}{c}{GM} \\ \hline
\multicolumn{2}{c|}{Approach} & Linear & Square & LM\_MLP & Linear & Square & LM\_MLP & Linear & Square & LM\_MLP & Linear & Square & LM\_MLP \\ \hline
\multicolumn{1}{c|}{\multirow{4}{*}{Abalone}} & \multirow{2}{*}{q=30} & 6.41 & 6.77 & \textbf{4.66} & 1.88 & 1.98 & \textbf{1.50} & 0.63 & 0.63 & \textbf{0.74} & 1.17 & 1.28 & \textbf{0.91} \\
\multicolumn{1}{c|}{} &  & (0.31) & (0.82) & \textbf{(0.49)} & (0.03) & (0.15) & \textbf{(0.05)} & (0.02) & (0.04) & \textbf{(0.02)} & (0.03) & (0.09) & \textbf{(0.07)} \\
\multicolumn{1}{c|}{} & \multirow{2}{*}{q=40} & 7.03 & 6.99 & \textbf{4.81} & 2.01 & 2.02 & \textbf{1.53} & 0.61 & 0.62 & \textbf{0.75} & 1.36 & 1.37 & \textbf{0.90} \\
\multicolumn{1}{c|}{} &  & (0.79) & (0.69) & \textbf{(0.42)} & (0.17) & (0.12) & \textbf{(0.08)} & (0.02) & (0.01) & \textbf{(0.01)} & (0.13) & (0.13) & \textbf{(0.04)} \\ \hline
\multicolumn{1}{c|}{\multirow{4}{*}{Airfoil}} & \multirow{2}{*}{q=30} & 18.74 & 17.98 & \textbf{16.72} & 3.33 & 3.25 & \textbf{3.09} & 0.80 & 0.80 & \textbf{0.81} & 2.15 & 2.03 & \textbf{1.93} \\
\multicolumn{1}{c|}{} &  & (3.41) & (1.70) & \textbf{(3.42)} & (0.30) & (0.14) & \textbf{(0.35)} & (0.05) & (0.04) & \textbf{(0.04)} & (0.22) & (0.12) & \textbf{(0.27)} \\
\multicolumn{1}{c|}{} & \multirow{2}{*}{q=40} & 23.61 & 19.57 & \textbf{18.31} & 3.72 & 3.46 & \textbf{3.24} & 0.74 & 0.78 & \textbf{0.80} & 2.38 & 2.25 & \textbf{2.08} \\
\multicolumn{1}{c|}{} &  & (3.20) & (1.98) & \textbf{(2.63)} & (0.26) & (0.16) & \textbf{(0.27)} & (0.03) & (0.02) & \textbf{(0.03)} & (0.25) & (0.16) & \textbf{(0.22)} \\ \hline
\multicolumn{1}{c|}{\multirow{4}{*}{Auto-mpg}} & \multirow{2}{*}{q=30} & 17.82 & 15.51 & \textbf{9.63} & 3.16 & 2.91 & \textbf{2.19} & 0.84 & 0.87 & \textbf{0.92} & 1.95 & 1.81 & \textbf{1.36} \\
\multicolumn{1}{c|}{} &  & (2.72) & (2.23) & \textbf{(1.62)} & (0.25) & (0.25) & \textbf{(0.18)} & (0.04) & (0.01) & \textbf{(0.01)} & (0.23) & (0.25) & \textbf{(0.16)} \\
\multicolumn{1}{c|}{} & \multirow{2}{*}{q=40} & 19.07 & 25.29 & \textbf{11.11} & 3.24 & 3.78 & \textbf{2.31} & 0.85 & 0.80 & \textbf{0.92} & 2.00 & 2.50 & \textbf{1.40} \\
\multicolumn{1}{c|}{} &  & (4.45) & (5.37) & \textbf{(3.03)} & (0.36) & (0.44) & \textbf{(0.30)} & (0.03) & (0.05) & \textbf{(0.01)} & (0.25) & (0.38) & \textbf{(0.12)} \\ \hline
\multicolumn{1}{c|}{\multirow{4}{*}{Housing}} & \multirow{2}{*}{q=30} & 32.85 & 32.00 & \textbf{22.13} & 3.94 & 3.78 & \textbf{3.47} & 0.80 & 0.82 & \textbf{0.86} & 2.36 & 2.31 & \textbf{2.08} \\
\multicolumn{1}{c|}{} &  & (8.39) & (9.79) & \textbf{(3.71)} & (0.32) & (0.38) & \textbf{(0.57)} & (0.05) & (0.05) & \textbf{(0.04)} & (0.12) & (0.13) & \textbf{(0.32)} \\
\multicolumn{1}{c|}{} & \multirow{2}{*}{q=40} & 30.67 & 31.79 & \textbf{24.53} & 3.99 & 4.11 & \textbf{3.49} & 0.82 & 0.81 & \textbf{0.85} & 2.56 & 2.61 & \textbf{2.22} \\
\multicolumn{1}{c|}{} &  & (5.06) & (5.00) & \textbf{(8.00)} & (0.49) & (0.38) & \textbf{(0.48)} & (0.04) & (0.04) & \textbf{(0.05)} & (0.51) & (0.32) & \textbf{(0.31)} \\ \hline
\multicolumn{1}{c|}{\multirow{4}{*}{Concrete}} & \multirow{2}{*}{q=70} & 106.74 & 105.08 & \textbf{58.45} & 8.02 & 8.00 & \textbf{5.86} & 0.80 & 0.80 & \textbf{0.89} & 5.21 & 5.16 & \textbf{3.68} \\
\multicolumn{1}{c|}{} &  & (8.94) & (14.70) & \textbf{(3.52)} & (0.33) & (0.53) & \textbf{(0.41)} & (0.03) & (0.02) & \textbf{(0.01)} & (0.26) & (0.50) & \textbf{(0.32)} \\
\multicolumn{1}{c|}{} & \multirow{2}{*}{q=80} & 107.80 & 109.25 & \textbf{59.47} & 8.01 & 8.03 & \textbf{5.71} & 0.80 & 0.80 & \textbf{0.89} & 5.25 & 5.24 & \textbf{3.57} \\
\multicolumn{1}{c|}{} &  & (13.61) & (11.09) & \textbf{(7.81)} & (0.41) & (0.51) & \textbf{(0.34)} & (0.04) & (0.02) & \textbf{(0.02)} & (0.25) & (0.47) & \textbf{(0.27)}
\\
\bottomrule
\end{tabular}
}
}
\end{table*}
\begin{table*}[!t]
\caption{Comparison of AVGL\_MSE and AVGL\_MAE}
\label{avglmse_avgl_mae}
\resizebox{1.00\textwidth}{!}{
\setlength{\tabcolsep}{3.5mm}{
\begin{tabular}{cc|cc|cc|cc|cc}
\toprule
\multicolumn{2}{c|}{Metric}                                              & \multicolumn{2}{c|}{MSE}    & \multicolumn{2}{c|}{MAE}          & \multicolumn{2}{c|}{Pearson}      & \multicolumn{2}{c}{GM}      \\ \hline
\multicolumn{2}{c|}{Approach}                                             & AVGL\_MSE & AVGL\_MAE       & AVGL\_MSE       & AVGL\_MAE       & AVGL\_MSE       & AVGL\_MAE       & AVGL\_MSE & AVGL\_MAE       \\ \hline
\multicolumn{1}{c|}{\multirow{4}{*}{Abalone}}     & \multirow{2}{*}{$q=30$} & 6.41      & \textbf{4.66}   & 1.94            & \textbf{1.50}   & 0.73            & \textbf{0.74}   & 1.04      & \textbf{0.91}   \\
\multicolumn{1}{c|}{}                             &                       & (0.43)    & \textbf{(0.49)} & (0.06)          & \textbf{(0.05)} & (0.03)          & \textbf{(0.02)} & (0.05)    & \textbf{(0.07)} \\
\multicolumn{1}{c|}{}                             & \multirow{2}{*}{$q=40$} & 7.77      & \textbf{4.81}   & 1.96            & \textbf{1.53}   & 0.73            & \textbf{0.75}   & 1.18      & \textbf{0.90}   \\
\multicolumn{1}{c|}{}                             &                       & (0.68)    & \textbf{(0.42)} & (0.08)          & \textbf{(0.08)} & (0.01)          & \textbf{(0.01)} & (0.05)    & \textbf{(0.04)} \\ \hline
\multicolumn{1}{c|}{\multirow{4}{*}{Airfoil}}     & \multirow{2}{*}{$q=30$} & 19.24     & \textbf{16.72}  & 3.41            & \textbf{3.09}   & 0.78            & \textbf{0.81}   & 2.21      & \textbf{1.93}   \\
\multicolumn{1}{c|}{}                             &                       & (1.70)    & \textbf{(3.42)} & (0.24)          & \textbf{(0.35)} & (0.04)          & \textbf{(0.04)} & (0.24)    & \textbf{(0.27)} \\
\multicolumn{1}{c|}{}                             & \multirow{2}{*}{$q=40$} & 24.37     & \textbf{18.31}  & 3.99            & \textbf{3.24}   & 0.77            & \textbf{0.80}   & 2.74      & \textbf{2.08}   \\
\multicolumn{1}{c|}{}                             &                       & (1.12)    & \textbf{(2.63)} & (0.11)          & \textbf{(0.27)} & (0.01)          & \textbf{(0.03)} & (0.21)    & \textbf{(0.22)} \\ \hline
\multicolumn{1}{c|}{\multirow{4}{*}{Auto-mpg}}    & \multirow{2}{*}{$q=30$} & 11.63     & \textbf{9.63}   & 2.50            & \textbf{2.19}   & \textbf{0.92}   & \textbf{0.92}   & 1.75      & \textbf{1.36}   \\
\multicolumn{1}{c|}{}                             &                       & (1.49)    & \textbf{(1.62)} & (0.29)          & \textbf{(0.18)} & \textbf{(0.01)} & \textbf{(0.01)} & (0.17)    & \textbf{(0.16)} \\
\multicolumn{1}{c|}{}                             & \multirow{2}{*}{$q=40$} & 20.08     & \textbf{11.11}  & 3.43            & \textbf{2.31}   & 0.91            & \textbf{0.92}   & 2.42      & \textbf{1.40}   \\
\multicolumn{1}{c|}{}                             &                       & (5.51)    & \textbf{(3.03)} & (0.49)          & \textbf{(0.30)} & (0.02)          & \textbf{(0.01)} & (0.43)    & \textbf{(0.12)} \\ \hline
\multicolumn{1}{c|}{\multirow{4}{*}{Housing}}     & \multirow{2}{*}{$q=30$} & 25.69     & \textbf{22.13}  & \textbf{3.44}   & 3.47            & 0.85            & \textbf{0.86}   & 2.27      & \textbf{2.08}   \\
\multicolumn{1}{c|}{}                             &                       & (10.69)   & \textbf{(3.71)} & \textbf{(0.43)} & (0.57)          & (0.05)          & \textbf{(0.04)} & (0.16)    & \textbf{(0.32)} \\
\multicolumn{1}{c|}{}                             & \multirow{2}{*}{$q=40$} & 31.74     & \textbf{24.53}  & 4.07            & \textbf{3.49}   & 0.83            & \textbf{0.85}   & 2.55      & \textbf{2.22}   \\
\multicolumn{1}{c|}{}                             &                       & (4.97)    & \textbf{(8.00)} & (0.39)          & \textbf{(0.48)} & (0.03)          & \textbf{(0.05)} & (0.43)    & \textbf{(0.31)} \\ \hline
\multicolumn{1}{c|}{\multirow{4}{*}{Concrete}}    & \multirow{2}{*}{$q=70$} & 69.04     & \textbf{58.45}  & 6.55            & \textbf{5.86}   & 0.88            & \textbf{0.89}   & 4.26      & \textbf{3.68}   \\
\multicolumn{1}{c|}{}                             &                       & (7.92)    & \textbf{(3.52)} & (0.34)          & \textbf{(0.41)} & (0.02)          & \textbf{(0.01)} & (0.38)    & \textbf{(0.32)} \\
\multicolumn{1}{c|}{}                             & \multirow{2}{*}{$q=80$} & 86.05     & \textbf{59.47}  & 7.29            & \textbf{5.71}   & 0.86            & \textbf{0.89}   & 5.02      & \textbf{3.57}   \\
\multicolumn{1}{c|}{}                             &                       & (7.66)    & \textbf{(7.81)} & (0.26)          & \textbf{(0.34)} & (0.03)          & \textbf{(0.02)} & (0.14)    & \textbf{(0.27)} \\ \hline
\multicolumn{1}{c|}{\multirow{4}{*}{Power-plant}} & \multirow{2}{*}{$q=60$} & 31.70     & \textbf{23.54}  & 4.52            & \textbf{3.84}   & 0.95            & \textbf{0.96}   & 3.26      & \textbf{2.62}   \\
\multicolumn{1}{c|}{}                             &                       & (2.89)    & \textbf{(1.41)} & (0.19)          & \textbf{(0.11)} & (0.00)          & \textbf{(0.00)} & (0.19)    & \textbf{(0.03)} \\
\multicolumn{1}{c|}{}                             & \multirow{2}{*}{$q=70$} & 47.59     & \textbf{24.91}  & 5.55            & \textbf{3.96}   & 0.95            & \textbf{0.96}   & 4.07      & \textbf{2.76}   \\
\multicolumn{1}{c|}{}                             &                       & (2.86)    & \textbf{(0.60)} & (0.13)          & \textbf{(0.05)} & (0.00)          & \textbf{(0.00)} & (0.29)    & \textbf{(0.14)}
\\
\bottomrule
\end{tabular}
}
}
\end{table*}
\subsection{Complete Results on UCI Benchmark Datasets}
Table \ref{cmlp} and Table \ref{clinear} show the mean squared error with standard deviation on the test set using the MLP model and the linear model, respectively. We evaluate each method using MSE, MAE, Pearson correlation and GM. We repeat the sampling-and-training process 5 times. As the table shows, our proposed method LM has significant advantages in all evaluation metrics. In particular, by comparing the experimental results reported in Table \ref{cmlp} and Table \ref{clinear} we can observe that training with the MLP model is generally better than training with the linear model. This observation is consistent with the common knowledge that MLP models are more powerful than linear models. It is worth noting that we also compare with the maximum margin interval trees method (MMIT\cite{drouin2017maximum}), which is similar to our limiting method. They both want to limit the predicted values to the interval. Table \ref{mmit} shows the results of our method with the MLP model and MMIT with linear (Linear) and squared (Square) hinge loss variants.
\subsection{Comparison between AVGL\_MSE and AVGL\_MAE}
\label{comavgl}
Table \ref{avglmse_avgl_mae} shows the comparison between AVGL\_MSE and AVGL\_MAE on the used UCI benchmark datasets with the MLP model, where AVGL\_MSE and AVGL\_MAE substitute the regression loss function in the average loss method with MSE and MAE, respectively. As shown in the table, AVGL\_MAE is significantly better than AVGL\_MSE.
\begin{table*}[ht]
\centering
\caption{Complete evaluation results on UCI benchmark datasets with MLP model}
\label{cmlp}
\resizebox{1.00\textwidth}{!}{
\setlength{\tabcolsep}{4.0mm}{
\begin{tabular}{c|ccc|ccc|ccc|ccc|cccc|c}
\hline \hline
\multirow{2}{*}{Datasets}     & \multirow{2}{*}{Metrese} & \multirow{2}{*}{$q$}  & \multirow{2}{*}{Supervised} & \multicolumn{3}{c|}{Leftmost}   & \multicolumn{3}{c|}{Rightmost}   & \multicolumn{3}{c|}{Middlemost}                            & \multirow{2}{*}{CRM} & \multirow{2}{*}{RANN} & \multirow{2}{*}{SINN} & \multirow{2}{*}{IN} & \multirow{2}{*}{LM} \\
                              &                          &                     &                             & MSE     & MAE     & Huber   & MSE      & MAE     & Huber   & MSE             & MAE             & Huber           &                      &                       &                       &                     &                     \\ \hline
\multirow{16}{*}{Abalone}     & \multirow{4}{*}{MSE}     & \multirow{2}{*}{30} &                             & 65.33   & 55.73   & 59.82   & 7.99     & 8.08    & 8.03    & 6.38            & 5.45            & 6.01            & 6.51                 & 6.67                  & 6.44                  & 5.61                & \textbf{4.66}       \\
                              &                          &                     & 4.57                        & (2.09)  & (1.79)  & (3.07)  & (0.70)   & (0.71)  & (0.71)  & (0.44)          & (0.43)          & (0.44)          & (0.52)               & (0.48)                & (0.44)                & (0.40)              & \textbf{(0.49)}     \\
                              &                          & \multirow{2}{*}{40} & (0.50)                      & 90.12   & 85.71   & 89.27   & 8.13     & 8.23    & 8.21    & 7.77            & 7.81            & 7.79            & 7.85                 & 7.74                  & 7.85                  & 7.84                & \textbf{4.81}       \\
\hline
\hline
\end{tabular}
}
}
\end{table*}
\begin{table*}[!t]
\caption{Complete evaluation results on UCI benchmark datasets with Linear model}
\label{clinear}
\resizebox{1.00\textwidth}{!}{
\setlength{\tabcolsep}{3.5mm}{
\begin{tabular}{c|ccc|ccc|ccc|ccc|cccc|c}
\hline \hline
\multirow{2}{*}{Datasets}     & \multirow{2}{*}{Metrese} & \multirow{2}{*}{$q$}  & \multirow{2}{*}{Supervised} & \multicolumn{3}{c|}{Leftmost}                           & \multicolumn{3}{c|}{Rightmost}                          & \multicolumn{3}{c|}{Middlemost}                            & \multirow{2}{*}{CRM} & \multirow{2}{*}{RANN} & \multirow{2}{*}{SINN} & \multirow{2}{*}{IN} & \multirow{2}{*}{LM} \\
                              &                          &                     &                             & MSE             & MAE             & Huber           & MSE             & MAE             & Huber           & MSE             & MAE             & Huber           &                      &                       &                       &                     &                     \\ \hline
\multirow{16}{*}{Abalone}     & \multirow{4}{*}{MSE}     & \multirow{2}{*}{30} &                             & 67.82           & 60.42           & 64.36           & 8.11            & 8.26            & 8.29            & 6.63            & 5.65            & 6.13            & 6.65                 & 6.70                  & 6.65                  & 5.70                & \textbf{4.92}       \\
                              &                          &                     & 4.87                        & (3.53)          & (4.52)          & (3.58)          & (0.68)          & (0.59)          & (0.59)          & (0.44)          & (0.48)          & (0.44)          & (0.45)               & (0.45)                & (0.45)                & (0.49)              & \textbf{(0.50)}     \\
\hline
\hline
\end{tabular}
}
}
\end{table*}

\end{document}